\documentclass[11pt]{article} 
\usepackage{hyperref}
\usepackage{url}
\usepackage{amsmath,amssymb,amsthm}
\usepackage{hyperref}
\usepackage[margin=2.8cm]{geometry}
\usepackage{eucal}
\usepackage{bm}
\usepackage{enumerate}
\usepackage{algorithm}
\usepackage{algpseudocode}
\usepackage{float}
\usepackage{graphicx}
\usepackage{psfrag}

\newfloat{algorithm}{t}{lop}

\newcommand{\reals}{\mathbb{R}}
\newcommand{\pr}{\mathbb{P}}
\newcommand{\eps}{\varepsilon}
\newcommand{\ex}{\mathbb{E}}
\newcommand{\Xb}{\mathbf{X}}
\newcommand{\Xbs}{\mathbf{X}_*}

\newcommand{\rhob}{\overline{\rho}}

\newcommand{\ical}{\mathcal{I}}

\newcommand{\nats}{\mathbb{N}}

\newcommand{\Et}{\widetilde{E}}

\newcommand{\ei}{e}
\newcommand{\gam}{\gamma}

\newcommand{\Zs}{Z_*}
\newcommand{\Xs}{X_*}
\newcommand{\pcal}{\mathcal{P}}
\newcommand{\mdim}{m}

\newcommand{\vnorm}[1]{\|#1\|}
\DeclareMathOperator{\lip}{Lip}
\newcommand{\Qf}{Q}

\newcommand{\Tf}{T}
\newcommand{\Tex}{\Tf_{\mathrm{ex}}}
\newcommand{\Tap}{\Tf_{\mathrm{ap}}}
\newcommand{\qf}{q}
\newcommand{\thb}{\bm{\theta}}
\newcommand{\yv}{y}
\newcommand{\yvt}{\widetilde{y}}
\newcommand{\wv}{w}
\newcommand{\wvt}{\widetilde{w}}
\newcommand{\hv}{\bm{h}}
\newcommand{\Rf}{\mathcal{R}}
\newcommand{\Mf}{\mathcal{M}}
\newcommand{\Pt}{\widetilde{P}}
\newcommand{\gamt}{\widetilde{\gam}}
\newcommand{\omgu}{\overline{\omega}}
\newcommand{\omgd}{\underline{\omega}}
\newcommand{\xcal}{\mathcal{X}}

\newcommand{\bid}{b}
\newcommand{\onenorm}[1]{\|#1\|_1}
\newcommand{\onemnorm}[1]{|\!|\!|#1|\!|\!|_1}
\newcommand{\mnorm}[1]{|\!|\!|#1|\!|\!|}
\newcommand{\infnorm}[1]{\|#1\|_\infty}
\newcommand{\infmnorm}[1]{|\!|\!|#1|\!|\!|_\infty}
\newcommand{\jacob}{J}
\newcommand{\Ft}{{\widetilde{F}}}
\newcommand{\ddim}{d}
\newcommand{\eb}{\bm{e}}

\newcommand{\Htmp}{H}
\newcommand{\Ktmp}{K}
\newcommand{\rb}{\overline{r}}
\newcommand{\ui}[1]{^{#1}}
\newcommand{\li}[1]{_{#1}}
\newcommand{\uii}[1]{^{(#1)}}
\newcommand{\Is}{I_*}
\newcommand{\sigs}{\sigma_*}
\newcommand{\dnorm}[1]{\|#1\|}

\newcommand{\Lf}{L}
\newcommand{\rv}{r}

\newcommand{\nn}{{n_0}}

\newcommand{\sgnorm}[1]{\|#1\|_{\psi_2}}

\newcommand{\lam}{\lambda}
\newcommand{\lams}{\lambda_*}

\newcommand{\gf}{g}

\newcommand{\Mcons}{M}
\newcommand{\nui}{\nu}
\newcommand{\Mbb}{\mathbb{M}}
\newcommand{\Vc}{\mathcal{V}}
\newcommand{\Yrv}{Y}
\newcommand{\Imin}{I_{\min}}
\newcommand{\sigmax}{\sigma_{\max}}
\newcommand{\KAP}{\kappa}
\newcommand{\pmult}{\odot\,}
\newcommand{\compos}{\circ}
\newcommand{\Krho}{K_\rho}
\newcommand{\Lrho}{L_\rho}
\newcommand{\Mcon}{D}
 \newcommand{\lipc}{L}

\newcommand{\gfb}{\bar{g}}
\newcommand{\rbar}{\bar{r}}
\newcommand{\sbar}{\bar{s}}
\newcommand{\iid}{\stackrel{iid}{\sim}}
\newcommand{\xt}{\widetilde{x}}
\newcommand{\xv}{x}
\newcommand{\yt}{\widetilde{y}}
\newcommand{\thbt}{\widetilde{\thb}}
\newcommand{\thbs}{\thb^*}
\newcommand{\onevec}{\bm{1}}

\newcommand{\mybar}[1]{\overline{#1}}
\newcommand{\gams}{\gamma^*}
\newcommand{\Ncons}{N}

\newcommand{\dagg}[1]{{#1}^\dagger}
\newcommand{\thbb}{\dagg{\thb}}
\newcommand{\pthet}{{\eta}}

\newtheorem{thm}{Theorem}
\newtheorem{prop}{Proposition}
\newtheorem{cor}{Corollary}
\newtheorem{lem}{Lemma}

\newtheorem{exa}{Example}

\title{Bayesian inference as iterated random functions \\with 
applications to sequential inference in graphical models\thanks{Part of this work is presented at the NIPS 2013 conference.}}

\author{
Arash A.~Amini
\\
\and
XuanLong  Nguyen \\
}

\begin{document}

\maketitle

\begin{abstract}
We propose a general formalism of iterated random functions
with semigroup property, under which exact and approximate Bayesian 
posterior updates can be viewed as specific instances. 
A convergence theory for iterated random functions 
is presented. As an application of the general theory we analyze
convergence behaviors of exact and approximate message-passing
algorithms that arise in a sequential change point detection
problem formulated via a latent variable directed graphical model. 
The sequential inference algorithm and its supporting theory are 
illustrated by simulated examples. 
\end{abstract}

\section{Introduction}
The sequential posterior updates play a central role in
many Bayesian inference procedures. As an example, in Bayesian
inference one is interested in the posterior probability
of variables of interest given the data observed sequentially up to 
a given time point. As a more specific example which provides the motivation
for this work, in a sequential change point detection 
problem~\cite{AS-78}, the key
quantity is the posterior probability that a change has
occurred given the data observed up to present time.
When the underlying probability model is complex, e.g., a large-scale
graphical model, the calculation of such quantities in a fast
and online manner is a formidable challenge. In such situations 
approximate inference
methods are required -- for graphical models, message-passing
variational inference algorithms present a viable option~\cite{Pearl88,Jordan-Statsci-04}.

In this paper we propose to treat Bayesian inference in a complex 
model as a specific instance of an abstract system of iterated random functions (IRF),
a concept that originally arises in the study of Markov chains 
and stochastic systems~\cite{DiaFre99}. The key technical
property of the proposed IRF formalism that enables the
connection to Bayesian inference under conditionally independent
sampling is the
\emph{semigroup} property, which shall be defined shortly in the 
sequel. It turns out that most exact and approximate Bayesian inference 
algorithms may be viewed as specific instances of an IRF system. 
The goal of this paper is to present a general convergence theory for 
the IRF with semigroup property.  The theory is then applied to the 
analysis of exact and approximate
message-passing inference algorithms, which arise in the context
of distributed sequential change point problems using latent variable
and directed graphical model
as the underlying modeling framework.

We wish to note a growing literature on message-passing and sequential
inference based on graphical modeling
~\cite{Kreidl-Willsky-07,Cetin-etal-06,NguAmiRaj12,frank12}. 
On the other hand, convergence and error analysis of 
message-passing algorithms in graphical models is quite rare and challenging,
especially for approximate algorithms, and they are
typically confined to the specific form of belief propagation 
(sum-product) algorithm ~\cite{ihler05b,ihler07b,Roosta-08}. To the best 
of our knowledge, there is no existing work on the analysis of
message-passing inference algorithms for calculating conditional
(posterior) probabilities for latent random variables present in 
a graphical model. While such an analysis is a byproduct of this
work, the viewpoint we put forward here that equates Bayesian posterior
updates to a system of iterated random functions with semigroup property
seems to be new and may be of general interest.

The paper is organized as follows.  In
Sections~\ref{sec:gen:iter:intro}--~\ref{sec:main:abs:result}, we
introduce the general IRF system and provide our main result on its
convergence. The proof is deferred to Section~\ref{sec:proof:abs}. As
an example of the application of the result, we will provide a
convergence analysis for an approximate sequential inference algorithm
for the problem of multiple change point detection using graphical
models. The problem setup and the results are discussed in
Section~\ref{sec:mcp:setup}. An auxiliary result needed for the change
point application is proved in Section~\ref{sec:proof:prop:alg:rep}
with some of the more technical aspects left to the appendices.

\section{Bayesian posterior updates as iterated random functions}
\label{sec:gen:iter:intro}
In this paper we shall restrict ourselves to multivariate distributions of binary random variables.
To describe the general iteration, let $\pcal_\ddim :=
\pcal(\{0,1\}^d)$ be the space of probability measures on
$\{0,1\}^d$. The iteration under consideration recursively produces a
random sequence of elements of $\pcal_\ddim$, starting from some
initial value. We think of $\pcal_\ddim$ as a subset of
$\reals^{2^d}$ equipped with the $\ell_1$ norm (that is, the total
variation norm for discrete probability measures). To simplify, let $\mdim := 2^d$, and for $x \in
\pcal_\ddim$, index its coordinates as $x =
(x\ui{0},\dots,x\ui{m-1})$. For  $\thb \in \reals_+^\mdim$, consider
the function $\qf_{\thb}: \pcal_\ddim \to \pcal_\ddim$, defined by
\begin{align}\label{eq:qf:def}
  \qf_{\thb}(x) := \frac{x \pmult \thb}{x^T \thb}
\end{align}
where $x^T \thb = \sum_i x\ui{i} \thb\ui{i}$ is the usual inner
product on $\reals^m$ and $x \pmult \thb$ is pointwise multiplication
with coordinates $[x \pmult \thb]\ui{i} := x\ui{i} \thb\ui{i}$, for
$i=0,1,\dots,\mdim-1$. This function models the prior-to-posterior
update according to the Bayes rule. One can think of $\thb$ as the
likelihood and $x$ as the prior distribution (or the posterior in the previous
stage) and $q_{\thb}(x)$ as the (new) posterior based on the two. The
division by $x^T \thb$ can be thought of as the division by the
marginal to make a valid probability vector. (See
Example~\ref{ex:classic:cp} below.)

We consider the following general iteration
\begin{align}\label{eq:itr:def}
  \begin{split}
    \Qf_n(x) &= \qf_{\thb_n}(\Tf(\Qf_{n-1}(x)), 
    \quad n \ge 1, \\
    \Qf_0(x) &= x,
  \end{split}
\end{align}
for some deterministic operator $\Tf : \pcal_\ddim \to \pcal_\ddim$
and an i.i.d. random sequence $\{\thb_n\}_{n\ge 1} \subset
\reals_+^\mdim$. By changing  operator $\Tf$, one obtains different
iterative algorithms.

Our goal is to find sufficient conditions on $T$ and $\{\thb_n\}$ for the
convergence of the iteration to an extreme point of $\pcal_\ddim$,
which without loss of generality is taken to be $\eb\uii{0} :=
(1,0,0,\dots,0)$. Standard techniques for proving the convergence of  iterated random
functions are usually based on showing some averaged-sense contraction
property for the iteration function~\cite{DiaFre99,Ste99,WuWoo00,WuSha04}, which in our case is
$q_{\thb_n}(\Tf(\cdot))$. See~\cite{Sten12} for a recent survey. These techniques are not
applicable to our problem since $q_{\thb_n}$ is not in general
Lipschitz, in any suitable sense, precluding
$q_{\thb_n}(\Tf(\cdot))$ from satisfying the aforementioned
conditions.

Instead, the functions $\{ q_{\thb_n}\}$ have another property which
can be exploited to prove convergence; namely, they form a semi-group
under pointwise multiplication,
\begin{align}\label{eq:semi:group:def}
    \qf_{\thb \pmult \thb'} = \qf_{\thb} \compos \qf_{\thb'}, \quad \thb,\thb' \in \reals_+^\mdim,
\end{align}
where $\compos$ denotes the composition of functions. If $T$ is the
identity, this property allows us to write $\Qf_n(x) =
q_{\pmult_{i=1}^n \thb_i}(x)$ --- this is nothing but the Bayesian
posterior update equation, under conditionally independent sampling, while
modifying $T$ results in an approximate Bayesian inference procedure.
Since after suitable normalization, $\pmult_{i=1}^n \thb_i$
concentrates around a deterministic quantity, by the i.i.d. assumption
on $\{\thb_i\}$, this representation helps in determining the limit of
$\{\Qf_n(x)\}$. The main result of this paper, summarized in
Theorem~\ref{thm:itr:conv}, is that the same conclusions can be
extended to general Lipschitz maps $T$ having the desired
fixed point. 


\section{General convergence theory}\label{sec:main:abs:result}
Consider a sequence $\{\thb_n\}_{n \ge 1} \subset \reals_+^\mdim$ of
i.i.d. random elements, where $\mdim = 2^\ddim$. Let $\thb_n = (\thb_n\ui{0},\thb_n\ui{1},\dots,\thb_n\ui{m-1})$ with $\thb_n\ui{0} = 1$ for all $n$, and 
    \begin{align}\label{eq:def:thb:n:s}
        \thb_n^* := \max_{i=1,2,\dots,m-1} \thb_n\ui{i}.
    \end{align}
The normalization $\thb_n\ui{0} = 1$ is convenient for showing
convergence to $\eb\uii{0}$. This is without loss of generality, since
$\qf_{\thb}$ is invariant to scaling of $\thb$, that is
$\qf_{\thb} = \qf_{\beta \thb}$ for any $\beta > 0$. 

Assume the sequence $\{ \log \thb_n^*\}$ to be i.i.d. sub-Gaussian
with mean $ \le -\Is < 0$ and sub-Gaussian norm $ \le \sigs \in
(0,\infty)$. The sub-Gaussian norm  can be taken to be the $\psi_2$
Orlicz norm (cf. \cite[Section 2.2]{vdvWel96}), which we denote by $\|
\cdot\|_{\psi_2}$. By definition $ \sgnorm{Y} := \inf\{ C > 0:\; \ex
\psi_2(|\Yrv|/C)\le 1\} $ where $\psi_2(x) := e^{x^2} - 1$.

  Let $\|\cdot\|$ denote the $\ell_1$ norm on $\reals^\mdim$.
  Consider the sequence $\{ \Qf_n(x) \}_{n \ge 0}$ defined
  in~\eqref{eq:itr:def} based on $\{\thb_n\}$ as above, an initial
  point $x =
  (x\ui{0},\dots,x\ui{\mdim-1}) \in \pcal_\ddim$ and a Lipschitz map $\Tf : \pcal_\ddim
  \to \pcal_\ddim$. Let $\lip_\Tf$ denote the Lipschitz constant of
  $\Tf$, that is
$
    \lip_\Tf := \sup_{x\neq y} \| T(x) - T(y)\|/ \|x-y\|.
$

  Our main result regarding iteration~\eqref{eq:itr:def} is the
  following.
\begin{thm}\label{thm:itr:conv}
     Assume that
     $\lipc := \lip_\Tf\le 1$ and that $\eb\uii{0}$ is a
     fixed point of $\Tf$. Then, for all $n \ge 0$, and $\eps > 0$,
    \begin{align}\label{eq:abs:rate}
    \| Q_n(x) - \eb\uii{0}\| \le 2 \frac{1-x\ui{0}}{x\ui{0}} \big(\lipc e^{-\Is + \eps}\big)^n 
    \end{align}
    with probability at least $1 - \exp(-c\,n\eps^2/\sigs^2)$, for some absolute constant $c > 0$. 
\end{thm}


  

The proof of Theorem~\ref{thm:itr:conv} is outlined in Section~\ref{sec:proof:abs}. Our
main application of the theorem will be to the study of
convergence of stopping rules for a distributed multiple change point
problem endowed with latent variable graphical models. Before stating that problem, let us consider the
classical (single) change point problem first, and show how the theorem can be
applied to analyze the convergence of the optimal Bayes rule.

\begin{exa}\label{ex:classic:cp}
In the classical Bayesian change point problem~\cite{AS-78}, 
one observes a
sequence $\{X^1,X^2,X^3\dots\}$ of independent data points whose
distributions change at some random time $\lambda$. More precisely, given
$\lambda = k$, $X^1,X^2,\dots,X^{k-1}$ are distributed according to
$g$, and $X^{k+1},X^{k+2},\dots$ according to $f$. Here, $f$ and $g$
are densities with respect to some underlying measure. One also
assumes a prior $\pi$ on $\lambda$, usually taken to be geometric. The goal
is to find a stopping rule $\tau$ which can predict $\lambda$ based
on the data points observed so far. It is
well-known that a rule based on thresholding the posterior
probability of $\lambda$ is optimal (in a Neyman-Pearson sense). To be
more specific, let $\Xb^n := (X^1,X^2,\dots,X^n)$ collect the data up
to time $n$ and let $\gamma^n[n] := \pr (\lambda \le n | \Xb^n)$ be the 
posterior probability of $\lambda$ having occurred before (or at) time
$n$. Then, the Shiryayev rule
\begin{align}\label{eq:Shir:rule}
  \tau := \inf \{ n \in \nats: \gam^n[n] \ge 1-\alpha\} 
\end{align}
is known to asymptotically have the least expected delay, among all
stopping rules with false alarm probability bounded by $\alpha$.
\end{exa}


Theorem~\ref{thm:itr:conv} provides a way to quantify how fast the posterior
$\gamma^n[n]$ approaches $1$, once the change point has occurred, hence
providing an estimate of the detection delay, even for finite number of
samples. We should note that our approach here is somewhat
independent of the classical techniques normally used for analyzing 
stopping rule~\eqref{eq:Shir:rule}. To cast the problem in the general
framework of \eqref{eq:itr:def}, let us introduce the binary variable
\mbox{$Z^n := 1\{\lambda \le n\}$}, where $1\{\cdot\}$ denotes the indicator of
an event. Let $\Qf_n$ be the (random) distribution of $Z^n$ given
$\Xb^{n}$, in other words,
\begin{align*}
  \Qf_n := \big( \pr(Z^n = 1|\Xb^n),\, \pr(Z^n = 0|\Xb^n) \big).
\end{align*}
Since $\gamma^n[n] = \pr(Z = 1|\Xb^n)$, convergence of $\gamma^n[n]$ to
$1$ is equivalent to the convergence of $Q_n$ to $\eb\uii{0} =
(1,0)$. We have 
\begin{align}\label{eq:seq:Bayes:simp}
  P(Z^n|\Xb^{n}) \propto_{Z^n} \; P(Z^n,X^n | \Xb^{n-1}) =
  P(X^n|Z^n) P(Z^n|\Xb^{n-1}).
\end{align}
Note that $P(X^n|Z^n = 1) = f(X^n)$ and $P(X^n| Z^n = 0) =
g(X^n)$. Let $\thb_n := \big( 1, \frac{g(X^n)}{f(X^n)} \big)$ and
\begin{align*}
  \Rf_{n-1} := \big( \pr(Z^n = 1|\Xb^{n-1}),\, \pr(Z^n = 0|\Xb^{n-1}) ).
\end{align*}
Then,~\eqref{eq:seq:Bayes:simp} implies that $Q_n$ can be obtained by
pointwise multiplication of $\Rf_{n-1}$ by $f(X^n) \thb_n$ and
normalization to make a probability vector. Alternatively, we can
multiply by $\thb_n$, since the procedure is scale-invariant, that is,
$\Qf_n = \qf_{\thb_n}(\Rf_{n-1})$ using
definition~\eqref{eq:qf:def}. It remains to express $\Rf_{n-1}$ in
terms of $\Qf_{n-1}$. This can be done by using the Bayes rule and the
fact that $P(\Xb^{n-1}|\lambda = k)$ is the same for $k \in
\{n,n+1,\dots\}$. In particular, after some algebra
(see Appendix~\ref{sec:proof:eq:gamma:classic:recur}), one arrives at
\begin{align}\label{eq:gamma:classic:recur}
  \gamma^{n-1}[n] = \frac{\pi(n)}{\pi [n-1]^c} + \frac{\pi[n]^c}{\pi [n-1]^c} \gamma^{n-1}[n-1],
\end{align}
where $\gamma^{k}[n] := \pr(\lambda \le n | \Xb^{k})$, $\pi(n)$ is the
prior on $\lambda$ evaluated at time $n$, and $\pi[k]^c :=
\sum_{i=k+1}^\infty \pi(i)$. For the geometric prior with parameter
$\rho \in [0,1]$, we have $\pi(n) := (1-\rho)^{n-1} \rho$ and $\pi[k]^c =
\rho^k$. The above recursion then simplifies to $\gamma^{n-1}[n] = \rho +
(1-\rho) \gamma^{n-1}[n-1]$. Expressing in terms of $\Rf_{n-1}$ and
$\Qf_{n-1}$, the recursion reads
\begin{align*}
  \Rf_{n-1} = T(Q_{n-1}), \quad \text{where}\; 
  T\Big( \Big({x_1 \atop x_0} \Big)\Big)
  = \rho \Big({1 \atop 0}\Big) + (1-\rho) \Big({x_1 \atop x_0}\Big).
\end{align*}
In other words, $T(x) = \rho \eb\uii{0} + (1-\rho)x$ for $x \in
\pcal_2$. 

Thus, we have shown that an iterative algorithm for computing $\gamma^n[n]$
(hence determining rule~\eqref{eq:Shir:rule}), can be expressed in the
form of~\eqref{eq:itr:def} for appropriate choices of $\{\thb_n\}$ and
operator $T$. Note that $T$ in this case is Lipschitz with constant
$1-\rho$ which is always guaranteed to be $\le
1$. 

We can now use Theorem~\ref{thm:itr:conv} to analyze the convergence
of $\gamma^n[n]$. Let us condition on $\lambda = k+1$, that is, we
assume that the change point has occurred at time $k+1$. Then, the sequence
$\{X^n\}_{n \ge k+1}$ is distributed according to $f$, and we have $\ex
\thb_n^* = \int f \log \frac{g}{f} = -I$, where $I$ is the KL
divergence between densities $f$ and $g$. Noting that $\|Q_n -
\eb\uii{0}\| = 2(1-\gamma^n[n])$, we immediately obtain the following corollary.

\begin{cor}
  Consider Example~\ref{ex:classic:cp} and assume that
  $\log(g(X)/f(X))$, where $X \sim f$, is sub-Gaussian with sub-Gaussian norm $\le
  \sigma$.
  Let $I := \int f \log \frac{f}{g}$. Then, conditioned on $\lambda =
  k+1$, we have for $n \ge 1$,
  \begin{align*}
    \big| \gamma^{n+k}[n+k] - 1 \big| \le 
    \big[(1-\rho) e^{-I+\eps}\big]^{n}
    \Big( \frac{1}{\gamma^{k}[k]} - 1\Big)
  \end{align*}
  with probability at least $1-\exp(-c\,n \eps^2/\sigma^2)$.

\end{cor}

\section{Multiple change point problem via latent variable graphical models}\label{sec:mcp:setup}

We now turn to our main application for Theorem~\ref{thm:itr:conv}, in
the context of a multiple change point problem. In~\cite{AmiNgu13},  graphical
model formalism is used to extend the classical
change point problem (cf. Example~\ref{ex:classic:cp}) to cases where multiple distributed
latent change points are present. Throughout this section, we will use this setup which we 
now briefly sketch.

One starts with a network $G = (V,E)$ of $d$ sensors or nodes, each
associated with a change point $\lambda_j$. Each node $j$ observes
a private sequence of measurements $\Xb_j = (X_j^1,X_j^2,\dots)$ which
undergoes a change in distribution at time $\lambda_j$, that is, 
\begin{align*}
  X_j^1,X_j^2,\dots,X_j^{k-1} \mid \lambda_j = k \; \iid \; g_j, \qquad
  X_j^k,X_j^{k+1},\dots \mid \lambda_j = k \; \iid \; f_j,
\end{align*}
for densities $g_j$ and $f_j$ (w.r.t. some underlying measure). Each
connected pair of nodes share an additional sequence of
measurements. For example, if nodes $s_1$ and $s_2$ are connected,
that is, $e = (s_1,s_2) \in E$, then they both observe $\Xb_e =
(X_e^1,X_e^2,\dots)$. The shared sequence undergoes a change in
distribution at some point depending on $\lambda_{s_1}$ and
$\lambda_{s_2}$. More specifically, it is assumed that the
earlier of the two change points causes a change in the shared
sequence, that is, the distribution of $\Xb_e$ conditioned on
$(\lambda_{s_1},\lambda_{s_2})$ only depends on $\lambda_e := \lambda_{s_1} \wedge
\lambda_{s_2}$, the minimum of the two, i.e.,
\begin{align*}
  X_e^1,X_e^2,\dots,X_e^{k} \mid \lambda_e = k\; \iid\; g_e, \qquad
  X_e^{k+1},X_e^{k+2},\dots \mid \lambda_e = k\; \iid\; f_e.
\end{align*}

Letting $\lams := \{\lambda_j\}_{j \in V}$ and $\Xbs^n =
\{\Xb_j^n,\Xb_e^n\}_{j \in V, e\in E}$, we can write the joint
density of all random variables as
\begin{align}\label{eq:joint:def}
  P(\lams,\Xbs^n) = \prod_{j \in V} \pi_j(\lam_j) \prod_{j \in
    V} P(\Xb_j^n | \lam_j) \prod_{\ei\, \in E} P(\Xb_\ei^n|\lambda_{s_1},\lambda_{s_2}).
\end{align}
where $\pi_j$ is the prior on $\lambda_j$, which we
assume to be geometric with parameter $\rho_j$.
Network $G$ induces a graphical model~\cite{Pearl88} which
encodes the factorization~\eqref{eq:joint:def} of the joint density. (cf. Fig.~\ref{fig:gm:paths})

Suppose now that each node $j$ wants to detect its change point
$\lambda_j$, with minimum expected delay, while maintaining a false
alarm probability at most $\alpha$. Inspired by the classical change
point problem, one is interested in computing the posterior probability that the
change point has occurred up to now, that is,
\begin{align}
  \gam_j^n[n] &:= \pr (\lam_j \le n \mid \Xbs^n)\label{eq:post:upto:n}.
\end{align}
The difference with the classical setting is the conditioning is done
on all the data in the network (up to time $n$). It is easy to verify that the natural
stopping rule
\begin{align*}
  \label{eq:tau:j:def}
  \tau_j = \inf \{ n \in \nats: \; \gam_j^n[n] \ge 1 -\alpha \}
\end{align*}
satisfy the false alarm constraint. It has also been shown that this rule is
asymptotically optimal in terms of expected detection delay. Moreover,
an algorithm based on the well-known sum-product~\cite{Pearl88} has
been proposed, which allows the nodes to compute their posterior
probabilities~\ref{eq:post:upto:n} by message-passing. The algorithm
is exact when $G$ is a tree, and scales linearly in the number of
nodes. More precisely, at time $n$, the computational complexity is
$O(nd)$. The drawback is the linear dependence on $n$, which makes the
algorithm practically infeasible if the change points model rare
events (where $n$ could grow large before detecting the change.)

In the next section, we propose an approximate message passing
algorithm which has computational complexity $O(d)$, at each time
step. This circumvents the drawback of the exact algorithm and allows
for indefinite run times. We then show how the theory developed in
Section~\ref{sec:main:abs:result} can be used to provide convergence guarantees for this
approximate algorithm, as well as the exact one.


\subsection{Fast approximate message-passing (MP)}\label{sec:approx:alg}
We now turn to an approximate message-passing algorithm which, at each
time step, has
computational complexity $O(d)$. The derivation is similar
to that used for the iterative algorithm in Example~\ref{ex:classic:cp}.
Let us define binary variables
\begin{align}
  Z_j^n = 1\{\lam_j \le n\}, \quad \Zs^n = (Z_1^n,\dots,Z_d^n).
\end{align}
The idea is to compute $P(\Zs^n|\Xbs^n)$ recursively based on
$P(\Zs^{n-1}|\Xbs^{n-1})$.  By Bayes rule,%
\begin{align}
  P(\Zs^n|\Xbs^n) \;\propto_{\Zs^n} 
  \;P(\Zs^{n},X_*^{n}|\Xbs^{n-1}) &= 
  P(\Xs^n | \Zs^{n} ) \, P(\Zs^n | \Xbs^{n-1}) 
  \notag \\
 &= \Big[ \prod_{j \in V} P(X_j^{n}|Z_j^n) 
 \prod_{\{i,j\} \in E} P(X_{ij}^{n} | Z_i^n,Z_j^n) \Big]
 \, P(\Zs^{n}|\Xbs^{n-1}), \label{eq:graph:mod:Z}
\end{align}
where we have used the fact that given $\Zs^n$, $\Xs^n$ is independent
of $\Xbs^{n-1}$. To simplify notation, let us extend the edge set to
$\Et := E \cup \{\{j\}: j \in V\}$. This allows us to treat the
private data of node $j$, i.e., $\Xb_j$, as shared data of
a self-loop in the extended graph $(V,\Et)$.
Let $u_e(z;\xi) := [g_e(\xi)]^{1 -z} [f_e(\xi])^z$ for $e \in \Et$, $z
\in \{0,1\}$. Then, for $i \neq j$,
\begin{align}
	P(X_j^{n}|Z_j^n) = u_j(Z_j^n;X_j^{n}), \quad 
	P(X_{ij}^{n}|Z_i^n,Z_j^n) = u_{ij}(Z_i^n \vee Z_j^n;X_{ij}^{n}). 
	\label{eq:u:potentials}
\end{align}
It remains to express $P(\Zs^n | \Xbs^{n-1})$ in terms of $P(\Zs^{n-1}
| \Xbs^{n-1})$. It is possible to do this, exactly, at a cost of
$O(2^{|V|})$. For brevity, we omit the exact expression. (See Lemma~\ref{prop:alg:rep}
for some details.) We term the algorithm that employs the exact
relationship, the ``exact algorithm''.

In practice, however, the exponential complexity makes the exact
recursion of little use for large networks. To obtain a fast algorithm (i.e.,
$O(\text{poly}(d)$), we instead take a mean-field type approximation:
\begin{align}\label{eq:mod:prior:Z}
  P(\Zs^n | \Xbs^{n-1}) \approx \prod_{j \in V} P(Z_j^n | \Xbs^{n-1})
  = \prod_{j \in V} \nu(Z_j^n; \gamma_j^{n-1}[n]),
\end{align}
where $\nu(z;\beta) := \beta^z(1-\beta)^{1-z}$. That is, we
approximate a multivariate distribution by the product of its
marginals. By an argument similar to that used to derive~\eqref{eq:gamma:classic:recur}, we can
obtain a recursion for the marginals,
\begin{align}\label{eq:approx:gam:recur}
  \gamma_j^{n-1}[n] = \frac{\pi_j(n)}{\pi_j[n-1]^c} +
  \frac{\pi_j[n]^c}{\pi_j[n-1]^c} \gam_j^{n-1}[n-1],
\end{align}
where we have used the notation introduced earlier in~\eqref{eq:gamma:classic:recur}.
Thus, at time $n$, the RHS of~(\ref{eq:mod:prior:Z}) is known based on
values computed at time $n-1$ (with initial value $\gam_j^0[0] =0, j
\in V$). Inserting this RHS
into~(\ref{eq:graph:mod:Z}) in place of $P(\Zs^n|\Xbs^{n-1})$, we
obtain a graphical model in variables $\Zs^n$ (instead of $\lams$) which has the same form
as~(\ref{eq:joint:def}) with $\nu(Z_j^n; \gamma_j^{n-1}[n])$ playing the
role of the prior $\pi(\lam_j)$.

In order to obtain the marginals $\gamma_j^n[n] = P(Z_j^n = 1|\Xbs^n)$ 
 with respect to the approximate 
version of the joint distribution $P(\Zs^{n},X_*^{n}|\Xbs^{n-1})$, we need to 
marginalize out the latent variables $Z_j^n$'s, for which a standard
sum-product algorithm can be applied (see~\cite{Pearl88,Jordan-Statsci-04,AmiNgu13}).
The message update equations are similar to those
in~\cite{AmiNgu13}; the difference is that the messages are
now binary and do not grow in size with $n$.  The approximate
algorithm is summarized in Algorithm~\ref{alg:approx}.

\begin{algorithm}
\caption{Message passing algorithm to
compute approximate posteriors $\gamt_j^n[n]$ and $\gamt_{ij}^n[n]$}
\label{alg:approx}

\begin{algorithmic}
   \State Initialize $\gamt_j^0[0] = 0$ for $j \in V$.
   \ForAll {time $n \ge 1$}
   \begin{enumerate}
     \item  Compute $\gamt_j^{n-1}[n]$ based on $\gamt_j^{n-1}[n-1]$
      using equation~\eqref{eq:approx:gam:recur}, for all $j \in V$.
   
    \item {Form the following joint distribution for $\Zs^n = (Z_1^n,\dots,Z_d^n)$,
      \begin{align}\label{eq:detailed:approx:joint}
        \Pt(\Zs^n|\Xbs^n) 
        = C  \prod_{j \in V }u_j(Z_j^n;X_j^{n})  
        \prod_{\{i,j\} \in E} u_{ij}(Z_i^n \vee Z_j^n;X_{ij}^{n})
        \, \prod_{j \in V} \nu(Z_j^n; \gamt_j^{n-1}[n])
      \end{align}
      where $u_e(z;\xi) := [g_e(\xi)]^{1 -z} [f_e(\xi])^z$ for $e \in
      \Et$, and $\nu(z;\beta) := \beta^z(1-\beta)^{1-z}$.  The
      normalizing constant $C$ is left undetermined at this
      point.}
  \item Invoke a message-passing algorithm (sum-product) on the joint
    distribution~\eqref{eq:detailed:approx:joint} to obtain marginal
    distributions
    $\Pt(Z_j^n|\Xbs^n)$, $j \in V$ and set $\gamt_j^n[n] =
    \Pt(Z_j^n=1|\Xbs^n)$.   
    
    (As a by-product of the message-passing, one also gets pair
    marginals $\Pt(Z_i^n,Z_j^n|\Xbs^n)$ and $\gamt_{ij}^n[n] :=
    \Pt(Z_i^n=1\; \text{or}\;Z_j^n=1|\Xbs^n)$ which are useful for
    constructing stopping rules for minimum of the two change points;
    see~\cite{AmiNgu13}.)
   \end{enumerate}
   
   \EndFor 
\end{algorithmic}



\end{algorithm}

\subsection{Convergence of MP algorithms}
\label{sec:mcp:results}
We now turn to the analysis of 
the approximate algorithm introduced in
Section~\ref{sec:approx:alg}. In particular, we will look at the
evolution of $\{\Pt(\Zs^n|\Xbs^n)\}_{n \in \nats}$ as a sequence of
probability distribution on $\{0,1\}^d$. Here, $\Pt$ signifies that
this sequence is an approximation. In order to make a meaningful
comparison, we also look at the algorithm which computes the exact
sequence $\{P(\Zs^n | \Xbs^n)\}_{n \in \nats}$, recursively. As
mentioned before, this we will call the ``exact algorithm'', the
details of which are not of concern to us at this point
(cf. Proposition~\ref{prop:alg:rep} for these details.) 


Recall that we take $\Pt(\Zs^n|\Xbs^n)$ and $P(\Zs^n|\Xbs^n)$, as distributions for $\Zs^n$, to be elements of $\pcal_\ddim \subset \reals^\mdim$. To make this correspondence formal and the notation simplified, we use the symbol $:\equiv$ as follows
\begin{align}\label{eq:equiv:notation}
    \yvt_n :\equiv \Pt(\Zs^n|\Xbs^n), \quad \yv_n :\equiv P(\Zs^n|\Xbs^n)
\end{align}
where now $\yvt_n,\yv_n \in \pcal_d$. Note that $\yvt_n$ and $\yv_n$
are random elements of $\pcal_d$, due the randomness of $\Xbs^n$. 
We have the following description.
\begin{prop}\label{prop:alg:rep}
  The exact and approximate sequences, $\{\yv_n\}$ and $\{\yvt_n\}$,
  follow general iteration~\eqref{eq:itr:def} with the same random
  sequence $\{\thb_n\}$, but with different deterministic operators
  $\Tf$, denoted respectively with
  $\Tex$ and $\Tap$. $\Tex$ is linear and given by a Markov transition
  kernel. $\Tap$ is a polynomial map of degree $d$. Both maps are
  Lipschitz and we have
  \begin{align}
    \lip_{\Tex} \le \Lrho:= \Big( 1 - \prod_{j=1}^d \rho_j \Big), \quad
    \lip_{\Tap} \le \Krho := \sum_{j=1}^d (1-\rho_j).
  \end{align}
\end{prop}
Detailed descriptions of the sequence $\{\thb_n\}$ and the operators
$\Tex$ and $\Tap$, along with the proof of Proposition~\ref{prop:alg:rep},
are given in Section~\ref{sec:proof:prop:alg:rep}. As suggested by
Theorem~\ref{thm:itr:conv}, a key assumption for the convergence of
the approximate algorithm will be $\Krho \le 1$. In contrast, we
always have $\Lrho \le 1$.

Recall that $\{\lam_j\}$ are the change points and their priors are
geometric with parameters $\{\rho_j\}$. We analyze the algorithms,
once all the change points have happened. More precisely, we condition on 
$$\Mbb_\nn := \{\max_j \lam_j \le \nn \}$$
for some $\nn \in \nats$. Then, one expects the (joint) posterior of $\Zs^n$ to contract to the point $Z_j^\infty = 1$, for all $j \in V$. In the vectorial notation, we expect both $\{\yvt_n\}$ and $\{\yv_n\}$ to converge to $\eb\uii{0}$. Theorem~\ref{thm:main:conv} below quantifies this convergence in $\ell_1$ norm (equivalently, total variation for measures).

Recall pre-change and post-change densities $g_e$ and $f_e$, and let
$I_e$ denote their KL divergence, that is, \mbox{$I_e := \int f_e \log
  (f_e/g_e)$}. We will assume that
\begin{align}\label{eq:Yrv:def}
\Yrv_e := \log(g_e(X)/f_e(X)) \quad \text{with} \quad X \sim f_e
\end{align} 
is sub-Gaussian, for all $e \in \Et$, where $\Et$ is  extended edge
notation introduced in Section~\ref{sec:approx:alg}. The choice $X \sim f_e$ is in accordance
with conditioning on $\Mbb_\nn$. Note that $\ex \Yrv_e = -I_e <
0$. We define
\begin{align*}
  \sigmax := \max_{e \in \Et} \sgnorm{Y_e}, \quad 
  \Imin := \min_{e \in \Et} I_e, \quad  \Is(\KAP) := \Imin - \KAP \,\sigmax
  \sqrt{\log \Mcon}. .
\end{align*}
where $\Mcon := |V| + |E|$. 
The following is our main result regarding
sequences~\eqref{eq:equiv:notation} produced by  the exact and  approximate algorithms.

\begin{thm}\label{thm:main:conv}
 There exists an absolute constant $\KAP > 0$, such that if
$\Is(\KAP) > 0$, 
the exact algorithm converges at least geometrically w.h.p., that is, for all $n \ge 1$,
\begin{align}
    \vnorm{\yv_{n+\nn} - \eb\uii{0}} 
    \le 2 \frac{1-\yv_{\nn}}{\yv_{\nn}} 
    \big(\Lrho e^{-\Is(\kappa) + \eps}\big)^n 
\end{align}
 with probability at least $1-\exp\big[{-c\,n\eps^2
  / (\sigmax^2 \Mcon^2 \log \Mcon )}\big]$, conditioned on $\Mbb_\nn$.
If in addition, $\Krho \le 1$, the approximate algorithm also
converges at least geometrically w.h.p., i.e., for all $n \ge 1$,
\begin{align}
    \vnorm{\yvt_{n + \nn} - \eb\uii{0}} 
    \le 2 \frac{1-\yvt_{\nn}}{\yvt_{\nn}}
    \big(\Krho e^{-\Is(\kappa) + \eps}\big)^n 
\end{align}
with the same (conditional) probability as the exact algorithm.
\end{thm}

\begin{proof}
  Proposition~\ref{prop:alg:rep}
and Theorem~\ref{thm:itr:conv} provide all the ingredients for the proof. 
It remains  to show that $\{\thb_n\}_{n \ge \nn}$ as given in~\eqref{eq:def:actual:thbn} satisfies the conditions of Theorem~\ref{thm:itr:conv}; namely, that $\{\log \thb_n^*\}_{n \ge \nn}$ is i.i.d. sub-Gaussian. We work conditioned on the event $\Mbb_\nn := \{\max_{j \in V} \lam_j \le \nn\}$, that is, we look at what happens to the  iterations past all the change-points. Throughout this section, $\ex$ denotes conditional expectation given $\Mbb_\nn$. Then, the fact that the sequence is i.i.d. follows immediately from the definition. Let us now focus on showing that $\log \thb_\nn^*$ is sub-Gaussian with negative expectation. We can write
\begin{align*}
    \log (\thb_\nn)_\ell = \sum_{e \in \Et} \nui_e^{\ell}\, \Yrv_e
\end{align*}
where $\Et$ is the extended edge notation introduced in
Section~\ref{sec:approx:alg}, $\Yrv_e := \log
[g_e(X_e^\nn)/f_e(X_e^\nn)]$, and $\nui_e^\ell \in \{0,1\}$. Note that
$\nui_e^\ell$ is equal to either $1-\bid_j(\ell)$ or $1-\bid_i(\ell)
\vee \bid_j(\ell)$. For $\ell \neq \mdim - 1$, at least one of
$\nui_e^\ell, e \in \Et$ is non-zero. From
definition~\eqref{eq:def:thb:n:s} and  superscript to subscript
index translation of~\eqref{eq:sub:sup:nota}, we have
\begin{align*}
    \log \thb_\nn^* = \max_{i = 1,2,\dots,\mdim-1} \log \thb_\nn^i
        = \max_{\ell = 0,1,\dots,\mdim-2} \log(\thb_\nn)_\ell.
\end{align*}
Let $\Vc \subset \{0,1\}^{|\Et|}$ denote the set carved by $(\nui_e^\ell)_{e \in \Et}$ as $\ell$ takes the values $0,1,\dots,\mdim-2$. We note that the all-zero vector does not belong to $\Vc$. Let $\nui = (\nui_e)_{e \in \Et}$ denote a generic point of $\{0,1\}^{|\Et|}$. Then, we have
\begin{align}
    \log \thb_\nn^* &= \max_{\nui \in \Vc} \sum_{e \in \Et} \nui_e \Yrv_e. \label{eq:log:thbs:majorant}
\end{align}
Note that $\ex \Yrv_e = \int f_e \log (g_e / f_e) = - I_e \le - \Imin$. We can write
\begin{align*}
    \ex \log \thb_\nn^*
    &\le \ex \Big[ \max_{\nui \in \Vc} \sum_{e \in \Et} \nui_e (\Yrv_e - \ex \Yrv_e) \Big] + \max_{\nui \in \Vc} \sum_{e \in \Et} \nui_e (\ex \Yrv_e) \\
    &\le  \ex \Big[ \max_{\nui \in \Vc} \sum_{e \in \Et} \nui_e |\Yrv_e - \ex \Yrv_e| \Big] + \max_{\nui \in \Vc} \sum_{e \in \Et} \nui_e (-\Imin).
\end{align*}
The second term above is equal to $-\Imin \big(\min_{\nui \in \Vc} \sum_{e \in \Et} \nui_e\big) = -\Imin$, due to the fact that at least one element of every $\nui \in \Vc$ is nonzero. Then, we have
\begin{align*}
    \ex \log \thb_\nn^* &\le \ex  \max_{\nui \in \Vc} \Big[\Big( \sum_{e \in \Et} \nui_e \Big) \max_{e \in \Et }|\Yrv_e - \ex \Yrv_e| \Big] - \Imin \\
        &\le | \Et| \,\ex \big(\max_{e \in \Et }|\Yrv_e - \ex \Yrv_e|\big) - \Imin
\end{align*}
We know that $\sgnorm{\Yrv_e - \ex \Yrv_e} \le c \sgnorm{\Yrv_e} \le c \,\sigmax$, for some numerical constant $c > 0$. In addition by majorant characteristic of $\psi_2$ space (cf.~\cite{vdvWel96,BulKoz00}),
\begin{align*}
    \ex \max_{e \in \Et }|\Yrv_e - \ex \Yrv_e| &\le C \sqrt{\log(1+|\Et|)} 
    \max_{e \in \Et} \sgnorm{\Yrv_e - \ex \Yrv_e} \\
        &\le C' \sqrt{\log(1+|\Et|)} \,\sigmax.
\end{align*}
Thus assuming $|\Et| \ge 2$, we have
\begin{align*}
    \ex \log \thb_\nn^* \le \KAP \sigmax\sqrt{\log |\Et|} - \Imin =: - \Is
\end{align*}
for some absolute constant $\KAP > 0$, which is the desired bound on the expectation of $\log \thb_\nn^*$.

To verify that $\log \thb_\nn^*$ is sub-Gaussian, we use $|\max a_i| \le \max |a_i|$ to write
\begin{align*}
    |\log \thb_\nn^*| \le \max_{\nui \in \Vc} \sum_{e \in \Et} \nui_e |\Yrv_e|
    \le |\Et| \max_{e \in \Et} |Y_e|.
\end{align*}
Since $\sgnorm{\cdot}$, as an Orlicz norm,  is monotone (i.e., $|X| \le |Y|$ implies $\sgnorm{X} \le \sgnorm{Y}$ for any two random variables $X$ and $Y$), we obtain
\begin{align*}
 \sgnorm{\log \thb_\nn^*} &\le |\Et| \cdot\sgnorm{\max_{e \in \Et} |Y_e|} \\
 &\le C |\Et| \sqrt{\log |\Et|} \max_{e \in \Et} \sgnorm{\Yrv_e} 
 \le C' \sigmax |\Et| \sqrt{\log |\Et|},
\end{align*}
where the second inequality is again by the majorant character of $\psi_2$. This completes the proof.
\end{proof}

\subsection{Simulation results}
We present some simulation results to verify the effectiveness of the
proposed approximation algorithm in estimating the posterior probabilities $\gam_j^n[n]$.
We consider a star graph on $d=4$ nodes. This is the subgraph on nodes
$\{1,2,3,4\}$ in Fig.~\ref{fig:gm:paths}. Conditioned on the change
points $\lams$, all data sequences $\Xbs$ are assumed Gaussian with
variance $1$, pre-change mean $1$ and post-change mean zero. All
priors are geometric with $\rho_j = 0.1$.  We note that
higher values of $\rho_j$ yield even faster convergence in the
simulations, but we omit these figures due to space constraints.  Fig.~\ref{fig:gm:paths}
illustrates typical examples of posterior paths $n \mapsto
\gamma_j^n[n]$, for both the exact and approximate MP algorithms.
One can observe that the approximate path often closely follows the
exact one. In some cases, they might deviate for a while, but as
suggested by Theorem~\ref{thm:main:conv}, they approach one another 
quickly, once the change points have occurred. 

From the theorem and triangle inequality, it follows that under
$\Is(\KAP) > 0$ and $\Krho \le 1$, $\vnorm{\yv_n - \yvt_n}$ converges
to zero, at least geometrically w.h.p. This gives some theoretical
explanation for the good tracking behavior of approximate algorithm as
observed in Fig.~\ref{fig:gm:paths}.

\begin{figure}[!t]
 \psfrag{Lambda1}{$\lambda_1$}
\psfrag{Lambda2}{$\lambda_2$}
\psfrag{Lambda3}{$\lambda_3$}
\psfrag{Lambda4}{$\lambda_4$}
\psfrag{Lambda5}{$\lambda_5$}
\psfrag{X1}{$\Xb_{12}$}
\psfrag{X2}{$\Xb_{23}$}
\psfrag{X3}{$\Xb_{24}$}
\psfrag{X4}{$\Xb_{45}$}
\psfrag{m12}{$m_{12}^{n}$}
\psfrag{m24}{$m_{24}^{n}$}
\psfrag{m32}{$m_{32}^{n}$}
\psfrag{m45}{$m_{45}^{n}$}
  \centering
  \includegraphics[width=.32\textwidth]{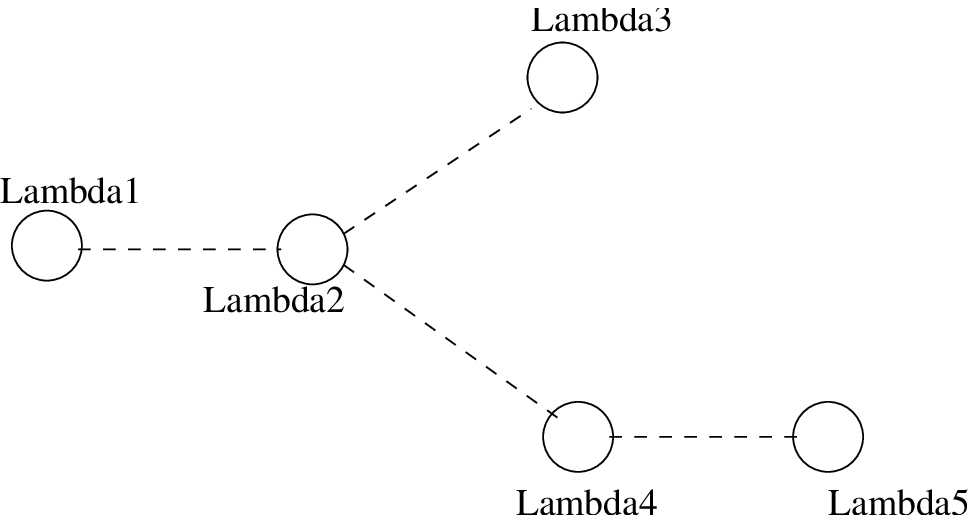} 
  \includegraphics[width=.32\textwidth]{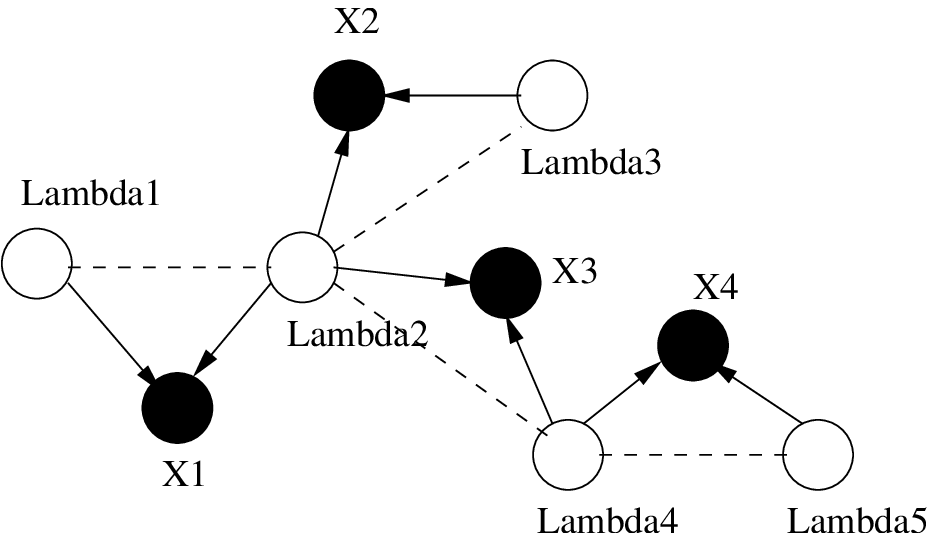} 
  \includegraphics[width=.32\textwidth]{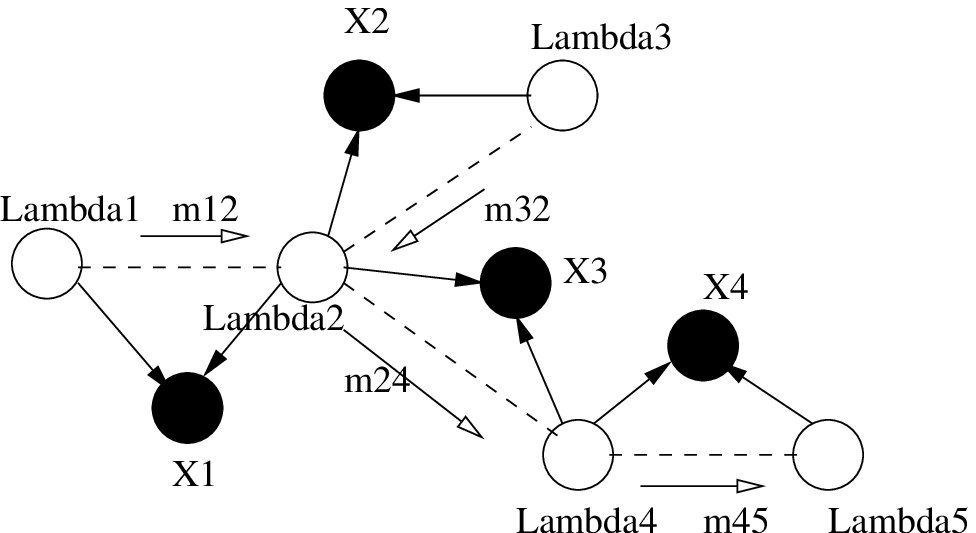}\\[2ex]
\includegraphics[width=1.5in]{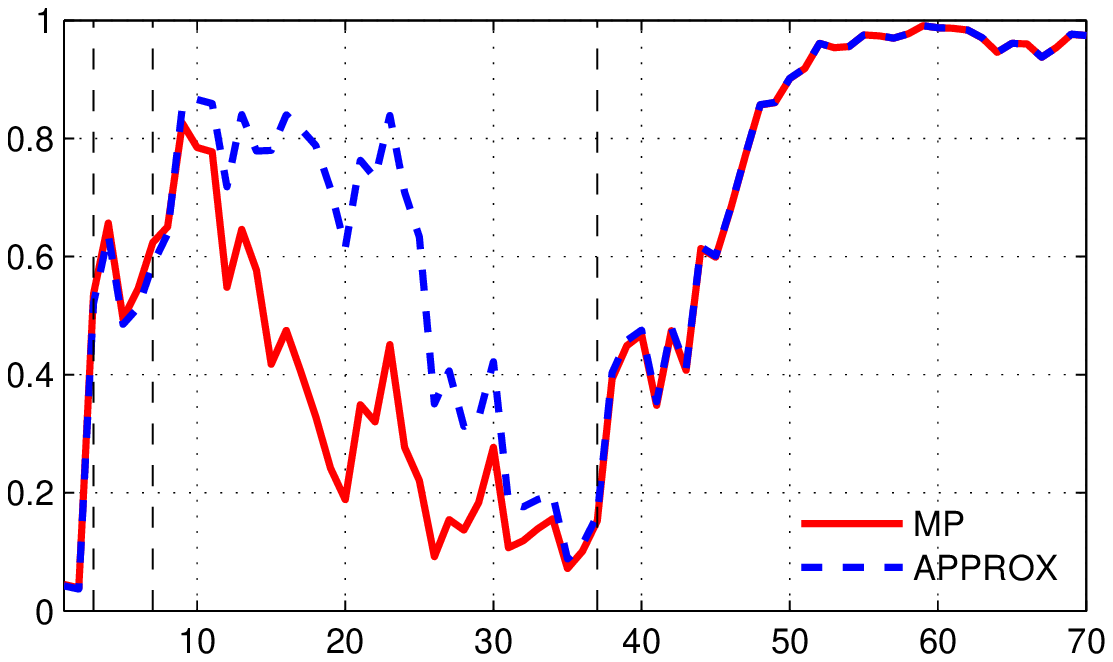}
\includegraphics[width=1.5in]{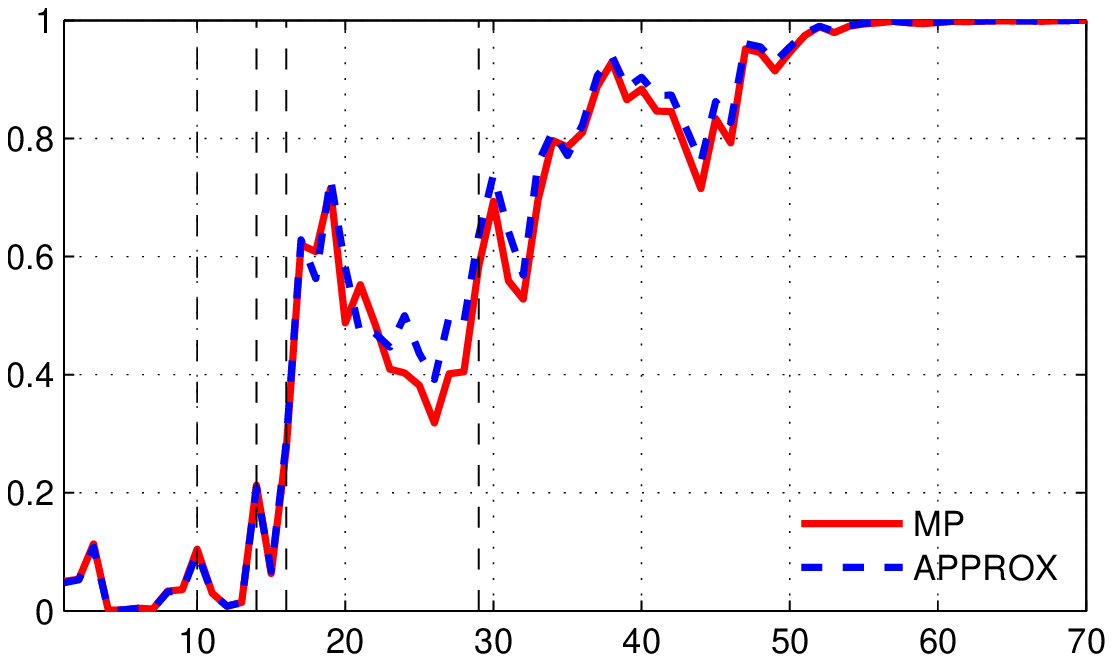}
\includegraphics[width=1.5in]{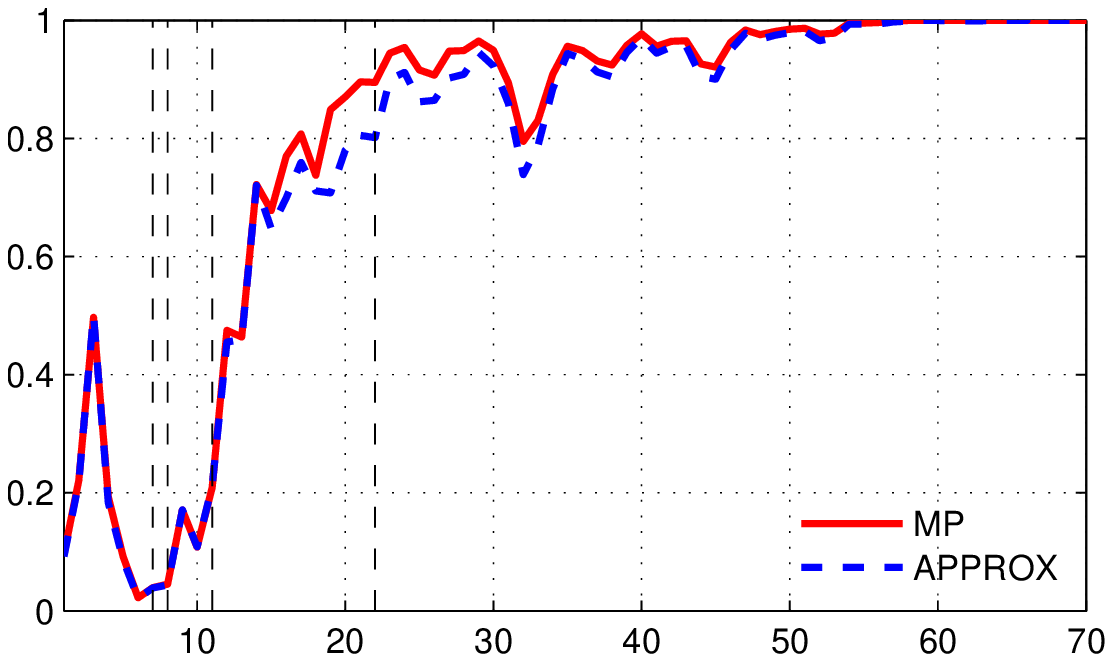}
\includegraphics[width=1.5in]{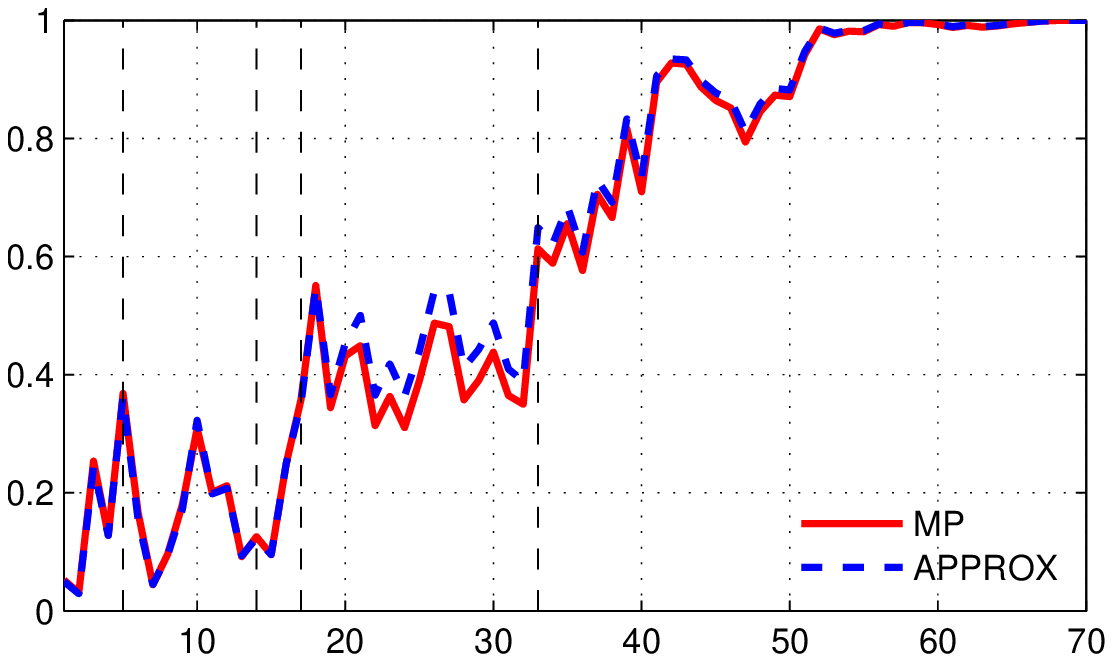}
\caption{ \small Top row illustrates a network (left), which induces a
  graphical model (middle). Right panel illustrates one stage of
  message-passing to compute posterior probabilities $\gam_j^n[n]$.
  Bottom row illustrates typical examples of posterior paths, $n
  \mapsto \gamma_j^n[n]$, obtained by EXACT and approximate (APPROX)
  message passing, for the subgraph on nodes $\{1,2,3,4\}$.  The
  change points are designated with vertical dashed lines.}

\label{fig:gm:paths}
\end{figure}


\section{Proof of Theorem~\ref{thm:itr:conv}}
\label{sec:proof:abs}




 For $x \in
\reals^\mdim$ (including $\pcal_\ddim$), we write
$x = (x\ui{0}, \xt)$ where $\xt = (x\ui{1},\dots,x\ui{\mdim-1})$.
Recall that $\eb\uii{0} = (1,0,\dots,0)$ and $\dnorm{x} =
\sum_{i=0}^{\mdim-1} |x_i|$. For $x \in \pcal_\ddim$, we have
$1 - x\ui{0} = \|\xt\|$, and 
\begin{align}\label{eq:dist:from:eb0}
  \dnorm{x - \eb\uii{0}} = \dnorm{(x\ui{0}-1,\xt)} = 1-x\ui{0} +
  \dnorm{\xt} = 2(1-x\ui{0}).
\end{align}
For $\thb =
(\thb\ui{0},\thbt) \in \reals_+^\mdim$, let 
\begin{align}\label{eq:thbs:thbb:defs}
  \thbs := \| \thbt\|_\infty = \max_{i=1,\dots,\mdim-1}
  \thb\ui{i}, \qquad 
  \thbb := \big(\thb\ui{0}, (\thbs \lipc) \onevec_{\mdim-1}\big) \in
  \reals_+^\mdim
\end{align}
where $\onevec_{\mdim-1}$ is a vector in $\reals^{\mdim - 1}$ whose
coordinates are all ones. We start by investigating how $\dnorm{\qf_{\thb}(x) - \eb\uii{0}}$
varies as a function of $\dnorm{x - \eb\uii{0}}$. 
\begin{lem}\label{lem:qf:dev}
  For $\lipc \le 1$, $\thbs > 0$, and $\thb\ui{0} = 1$,
  \begin{align}\label{eq:qf:dev}
    \Ncons := \sup_{\substack{\xv,\yv \,\in \,\pcal_\ddim, \\ 
        \dnorm{\xv - \eb\uii{0}} \,\le\, 
        \lipc \dnorm{\yv - \eb\uii{0}}}} \; 
    \frac{\| \qf_{\thb}(\xv) - \eb\uii{0}\|}{
      \dnorm{\qf_{\thbb}(\yv) - \eb\uii{0}}} = 1;
  \end{align}
\end{lem}

We prove Lemma~\ref{lem:qf:dev} shortly in
Section~\ref{sec:proof:lem:qf:dev}. 
Given the lemma, let us proceed to the proof of the theorem.  
Recall that $\Tf : \pcal_\ddim \to
\pcal_\ddim$ is an $\lipc$-Lipschitz map, and that $\eb\uii{0}$ is a
fixed point of $\Tf$, that is, $\Tf(\eb\uii{0}) = \eb\uii{0}$. It
follows that for any $\xv \in \pcal_\ddim$, $\dnorm{\Tf(\xv)
  -\eb\uii{0}} \le \lipc \dnorm{\xv - \eb\uii{0}}$. Applying
Lemma~\ref{lem:qf:dev}, we get
\begin{align}\label{eq:peel:ident}
  \dnorm{\qf_{\thb}(\Tf(\xv)) - \eb\uii{0}} \le 
  \| \qf_{\thbb}(\xv) - \eb\uii{0}\|
\end{align}
for $\thb \in \reals_+^\mdim$ with $\thb\ui{0} = 1$, and $\xv \in
\pcal_\ddim$. (This holds even if $\thbs = 0$ where  both sides are zero.)


Recall the sequence $\{\thb_n\}_{n \ge 1}$ used in
defining  functions $\{\Qf_n\}$ accroding to~\eqref{eq:itr:def}, and
the assumption that $\thb_n\ui{0} = 1$, for all $n \ge 1$. 
Inequality~\eqref{eq:peel:ident} is key in allowing us to
peel operator $T$, and bring successive elements of $\{q_{\thb_n}\}$
together. Then, we can exploit the semi-group
property~\eqref{eq:semi:group:def} on adjacent elements of $\{q_{\thb_n}\}$.

To see this, for each $\thb_n$, let $\thbs_n$ and $\thbb_n$ be defined
as in~\eqref{eq:thbs:thbb:defs}. Applying~\eqref{eq:peel:ident} with $x$
replaced with $\Qf_{n-1}(\xv)$, and $\thb$ with $\thb_n$, we can write
\begin{align*}
  \dnorm{ \Qf_n(\xv) - \eb\uii{0}} &\le 
  \dnorm{ \qf_{\thbb_n} (\Qf_{n-1}(\xv)) - \eb\uii{0} } \quad 
  \text{(by~(\ref{eq:peel:ident}))} \\
  &= 
  \| \qf_{\thbb_n} ( \qf_{\thb_{n-1}}(T( Q_{n-2}(\xv)))) - \eb\uii{0} \| \\
  &= 
  \| \qf_{\thbb_n \pmult \thb_{n-1}}(T( Q_{n-2}(\xv)))) - \eb\uii{0} \|\quad 
  \text{(by semi-group property~\eqref{eq:semi:group:def})} 
\end{align*}
We note that $(\thbb_n \pmult \thb_{n-1})^*
= \lipc \thbs_n \thbs_{n-1} $ and 
\begin{align*}
\dagg{(\thbb_n \pmult \thb_{n-1})} =
\big(1, \lipc (\thbb_n \pmult \thb_{n-1})^* \onevec_{\mdim-1} \big)
=\big(1, \lipc^2 \thbs_n \thbs_{n-1} \onevec_{\mdim-1} \big).
\end{align*}
Here, $*$ and $\dagger$ act on a general vector in the sense
of~\eqref{eq:thbs:thbb:defs}.
Applying~\eqref{eq:peel:ident} once more, we get
\begin{align*}
  \dnorm{ \Qf_n(\xv) - \eb\uii{0}} &\le \| \qf_{(1, \lipc^2
    \thbs_n \thbs_{n-1} \onevec_{\mdim-1} )}( Q_{n-2}(\xv)) -
  \eb\uii{0} \|.
\end{align*}
The pattern is clear. Letting $\pthet_n := \lipc^n \prod_{k=1}^n
\thbs_k$, we obtain by induction
\begin{align}\label{eq:mid:1}
   \dnorm{ \Qf_n(\xv) - \eb\uii{0}} \le 
   \| \qf_{(1, \pthet_n \onevec_{\mdim-1})}( Q_{0}(\xv)) -
  \eb\uii{0} \|.
\end{align}
Recall that $\Qf_0(x) := x$. Moreover,
\begin{align}\label{eq:mid:2}
  \| \qf_{(1, \pthet_n \onevec_{\mdim-1})}(x) - \eb\uii{0}\| = 2 \big(1-
  [\qf_{(1, \pthet_n \onevec_{\mdim-1})}(x)]\ui{0} \big) =
  2\big(1-\gf_{\pthet_n}(x\ui{0}) \big)
\end{align}
where the first inequality is by~\eqref{eq:dist:from:eb0}, and the
second is easily verified by noting that all the elements of
$(1,\pthet_{n}\onevec_{\mdim-1})$, except the first, are
equal. Putting~\eqref{eq:mid:1} and~\eqref{eq:mid:2} together with the
bound $1-\gf_{\theta}(r) = \frac{\theta(1-r)}{r+\theta(1-r)} \le
\theta \frac{1-r}{r}$, which holds  for
$\theta > 0$ and $r \in (0,1]$, we obtain
$ \dnorm{ \Qf_n(\xv) - \eb\uii{0}} \le 2 \pthet_n  \frac{1-x\ui{0}}{x\ui{0}}$.
By sub-Gaussianity assumption on $\{\log \thbs_k\}$, we have
\begin{align}\label{eq:thbs:subg:prob:bound}
  \pr \Big( \frac1n \sum_{k=1}^n \log \thbs_k - \ex \log \thbs_1 >
  \eps \Big) \le \exp(-c\, n\eps^2/\sigs^2),
\end{align}
for some absolute constant $c > 0$. (Recall that $\sigs$ is an upper
bound on the sub-Gaussian  norm $\sgnorm{\log \thbs_1}$.) On the
complement of the event in~\ref{eq:thbs:subg:prob:bound}, we have
$\prod_{k=1}^n \thbs_k \le e^{n(-\Is+\eps)}$, which completes the proof.

\subsection{Proof of Lemma~\ref{lem:qf:dev}}\label{sec:proof:lem:qf:dev}
We consider
the simplest case first, namely $\ddim = 2$. For $\theta \in
\reals_+$, let $\gf_\theta : [0,1] \to [0,1]$ be defined by
\begin{align}\label{eq:gf:def}
    \gf_\theta(\rv) := \frac{\rv}{\rv+\theta(1-\rv)}.
\end{align}
This function completely describes $\qf_{\thb}$ when $d = 2$. More
precisely, with $\thb = (1,\theta)$, one has $\qf_{\thb}(x) =
\big(\gf_\theta(x^0), 1-\gf_\theta(x^0)\big)$. Note that $\qf_{\thb}(x)$
is close to $\eb^{0}$ iff $\gf_\theta(x^0)$ is close to $1$. To simplify notation, let $\rbar := 1-r$  for $r \in
[0,1]$. Similarly, let
\begin{align}\label{eq:gfb:def}
  \gfb_\theta(r) := 1-\gf_\theta(r) = \frac{\theta \rbar}{1-\rbar + \theta
  \rbar}.
\end{align}
The next lemma allows us to quantify how $|\gfb_\theta(r)|$ varies in terms of
$|\rbar|$. Consider the following quantity
\begin{align}\label{eq:M:def}
    \Mcons_{\lipc}(\theta,\gam) := \sup \Big\{ 
     \frac{|\gfb_\theta(r)|}{|\gfb_\gamma(s)|}:\; \rbar, \sbar \in
     (0,1],\; \rbar \le \lipc \sbar \Big\}.
\end{align}

\begin{lem}\label{lem:M:bound}
  Assume that $\lipc \le 1$ and $\theta > 0$. Let $\eps:=
  1-\theta$ and $\gamma := 1 - \delta$. Then,
  \begin{align}\label{eq:M:bound}
    \Mcons_\lipc(\theta,\gamma) = \frac{\theta \lipc}{|\gamma|}
    \max\Big\{ 1, \Big| \frac{1-\delta}{1 - \lipc \eps}\Big| \Big\}.
  \end{align}
  In particular, for $\gamma = \theta \lipc$, we have
  $\Mcons_\lipc(\theta,\gamma) = 1$.
\end{lem}

\begin{proof}
  We can write
    \begin{align*}
        \Mcons_\lipc(\theta,\gam) &= \sup_{\rb,\sbar} \Big|
        \frac{\theta \rb}{1-\rb + \theta \rb}
        \frac{1-\sbar + \gam \sbar}{\gam \sbar} \Big| \\
        &= \frac{\theta}{|\gam|}\sup_{\rb,\sbar} \Big| \frac{\rb}{\sbar}  \cdot
        \frac{(\gam -1)\sbar+1}{(\theta -1)\rb + 1}\Big| \\
        &= \frac{\theta}{|\gam|}\sup_{\rb,\sbar} \Big|
        \frac{(\gam -1)+1/\sbar}{(\theta -1) + 1/\rbar}\Big| 
    \end{align*}
    Let $x = 1/\rb$ and $z = \rb/\sbar$. Then, the set $\{
    (\rbar,\sbar) :\; \rbar,\sbar \in (0,1],\; \rbar \le \lipc
    \sbar\}$ corresponds to
    \begin{align*}
      \{(x,z):\; x \ge 1, \; xz \ge 1,\; z \le \lipc \} 
      = \{(x,z):\; x \ge \frac1\lipc, \; \frac1{x} \le z \le \lipc \} 
    \end{align*}
    where in the second inequality, we used $\lipc \le 1$ and that
    $[\frac1{x},\lipc]$ is empty unless $x \ge
    \frac1{\lipc}$. Letting $m(x,z): = (xz - \delta) / (x-\eps)$, we obtain

    \begin{align*}
        \Mcons_\lipc(\theta,\gam) = \frac{\theta}{|\gam|} 
        \sup_{x \ge \frac1\lipc, \; z \in [\frac1{x},\lipc]}
        | m(x,z)|
    \end{align*}
    The function $m(x,z)$ is well-defined over the specified region
    (that is, finite-valued) since $\theta > 0$ implies $\eps < 1$,
    hence $x - \eps > 0$.
    For fixed $x \ge \frac1\lipc$, the function $z \mapsto |m(x,z)|$ is
    convex, hence achieving its maximum over the convex set
    $[\frac1x,\lipc]$, at one of the extreme points,
    \begin{align*}
       \Mcons_\lipc(\theta,\gam) = \frac{\theta}{|\gam|} 
        \sup_{x \ge \frac1\lipc} \Big[ \max\big\{ | m(x,\tfrac1x)|,
        |m(x,\lipc)| \big\} \Big]
    \end{align*}
    Both $x \mapsto |m(x,\frac1x)|$ and $x \mapsto |m(x,\lipc)|$ are
    quasi-convex, hence their suprema over $[\frac1\lipc,\infty)$ are
    obtained at one of the endpoints. Thus,
    \begin{align*}
       \Mcons_\lipc(\theta,\gam) &= \frac{\theta}{|\gam|} 
          \max\Big\{ \sup_{x \ge \frac1\lipc} 
          \Big| \frac{1-\delta}{x-\eps}\Big|,\;
        \sup_{x \ge \frac1\lipc}  \Big| \frac{x \lipc
          -\delta}{x-\eps}\Big|\Big\} \\
        &=\frac{\theta}{|\gam|} 
        \max \Big\{ \Big| \frac{1-\delta}{\frac1\lipc-\eps}\Big|, \, 0,\,
        \Big|  \frac{\lipc \frac1\lipc -\delta}{\frac1\lipc-\eps} \Big|, \,  \lipc\Big\}
    \end{align*}
    which simplifies to~\eqref{eq:M:bound}.


For the special case, $\gamma = \theta
\lipc$, we first note that $\lipc \theta / \gamma =
\gf_{1/\theta}(\lipc)$. Then, we have $M_{\lipc}(\theta,\gamma) =
\max\{1,\gf_{1/\theta}(\lipc)\}$. Since $\gf_{1/\theta}(\lipc) \in
[0,1]$, we get the desired result.
\end{proof}

Let us now move to the case of general $\ddim$. By~(\ref{eq:dist:from:eb0}), we have
  \begin{align}\label{eq:N:rewrite}
    \Ncons = \sup \Big\{
    \frac{1-[\qf_{\thb}(x)]\ui{0}}{1-[\qf_{\thbb}(y)]\ui{0} }
    :\;\; \mybar{x\ui{0}} \le \lipc  \mybar{y\ui{0}},\; \|\xt\|=
    \mybar{x\ui{0}},\; \| \yt\| = \mybar{y\ui{0}} \Big\}.
  \end{align}
  We are effectively optimizing over four variables $x\ui{0}, y\ui{0},
  \xt$ and $\yt$. Let us first optimize over $\xt$, fixing the other three.
  By definition~\eqref{eq:qf:def}, we have
  \begin{align*}
    \sup_{\xt: \;\|\xt\|=
    \mybar{x\ui{0}} } \big\{ 1-[\qf_{\thb}(x)]\ui{0} \}
   &= \sup_{\xt: \;\|\xt\|=
    \mybar{x\ui{0}} } \Big\{ 1- 
  \frac{\thb\ui{0} x\ui{0}}
       {\thb\ui{0} x\ui{0} + \thbt^T  \xt} \Big\} \\
   &= 1- 
  \frac{\thb\ui{0} x\ui{0}}
       {\thb\ui{0} x\ui{0} + \sup \big\{  \thbt^T  \xt :\; 
       \|\xt\|= \mybar{x\ui{0}} \}} 
    = 1- 
  \frac{\thb\ui{0} x\ui{0}}
       {\thb\ui{0} x\ui{0} +  \| \thbt\|_\infty\mybar{x\ui{0}} },
  \end{align*}
  by the duality of $\ell_1$ and $\ell_\infty$ norms. Recalling the
  definition~\eqref{eq:gf:def}, and using $\thb\ui{0} = 1$ and
  $\|\thbt\|_\infty = \thbs$, we have
  \begin{align}\label{eq:optim:xt}
    \sup_{\xt: \;\|\xt\|=
    \mybar{x\ui{0}} } \big\{ 1-[\qf_{\thb}(x)]\ui{0} \} = 1- \gf_{\thbs}(x\ui{0}).
  \end{align}
  Next, we optimize over $\yt$. Let $\gams := \thbs \lipc$. We note for $\|\yt\| =
  \mybar{y\ui{0}}$,
  \begin{align*}
    [\qf_{\thbb}(y)]\ui{0} = 
      \frac{\thb\ui{0} y\ui{0}}
          {\thb\ui{0} y\ui{0} + \gams \onevec_{m-1}^T  \yt} 
       = \frac{\thb\ui{0} y\ui{0}}
          {\thb\ui{0} y\ui{0} + \gams \mybar{y\ui{0}}} = \gf_{\gams}( y\ui{0})
  \end{align*}
  where we have used $\onevec_{m-1}^T  \yt =
  \|\yt\|$ and $\thb\ui{0} = 1$. In other words, we have shown
  \begin{align}\label{eq:optim:yt}
    \sup_{\yt :\; \|\yt\| = \mybar{y\ui{0}}} \big\{
    1-[\qf_{\thbb}(y)]\ui{0}\big\}
    = 1- \gf_{\gams}( y\ui{0}).
  \end{align}
  Substituting~\eqref{eq:optim:xt} and~\eqref{eq:optim:yt}
  in~\eqref{eq:N:rewrite}, and recalling the
  notation~\eqref{eq:gfb:def} and definition~(\ref{eq:M:def}), we get
  \begin{align*}
    N = \sup \Big\{
    \frac{\gfb_{\thbs}(x\ui{0})}{\gfb_{\gams}(y\ui{0})} :\;
    \mybar{x\ui{0}} \le \lipc  \mybar{y\ui{0}} \Big\} = M_{\lipc}(\thbs,\gams).
  \end{align*}
  Applying Lemma~\ref{lem:M:bound} in the special case $\gams = \thbs
  \lipc$, we get $M_{\lipc}(\thbs,\gams) = 1$ which gives the
  desired result.

\section{Proof of Proposition~\ref{prop:alg:rep}}
\label{sec:proof:prop:alg:rep}
We divide the proof into pieces with some of the more technical details
deferred to the Appendix. We will need some extra notations for
the indexing of coordinates of probability vectors in
$\pcal_\ddim = \pcal(\{0,1\}^\ddim)$. So far we have used superscripts to index the
coordinates from left to right. It is sometimes convenient to use a
complementary subscript indexing, by going from right to left. More
specifically, for $x \in \pcal_\ddim$, we write
\begin{align}\label{eq:sub:sup:nota}
  x = (x\ui{0},x\ui{1},\dots,x\ui{\mdim-1}) = 
  (x\li{\mdim-1}, \dots, x\li{1},x\li{0})
\end{align}
so that $x\li{i} = x\ui{\mdim-1-i}$. We also interpret $x\li{i}$ as
the value that $x$ assigns to the binary representation\footnote{For
  example, for $d = 2$, $x = (x\li3,x\li2,x\li1,x\li0) = \big(
  x(\{(1,1)\}), x(\{(1,0)\}), x(\{(0,1)\}), x(\{(0,0)\}\big)$, where
  the multitude of parentheses is because in the RHS, we are treating
  $x$ as a measure (i.e., a set-valued function) on all subsets of
  $\{0,1\}^d$.} of $i$. Furthermore, for any $i = 0,\dots,m-1$, let
\begin{align}
\label{eq:bin:rep:def}
	\bid_j(i) := 
        \text{$j$th bit from the left in binary expansion of $i$}, \quad j  \in [d],
\end{align}
so that the binary expansion of $i$ is the string $\bid_1(i)
\bid_2(i) \dots \bid_\ddim(i)$.

Before starting the proof, let us give an explicit expression for the
common sequence $\{\thb_n\}$ used in the iterations of both the exact
and approximate algorithms. Recall the notation $\yv_n :\equiv P(\Zs^n|\Xbs^n)$ introduced in~\eqref{eq:equiv:notation}, in which $\yv_n \in \pcal_\ddim$ is defined by looking at $P(\Zs^n|\Xbs^n)$ as a random probability vector indexed by $\Zs^n \in \{0,1\}^d$. Similarly, in view of ~\eqref{eq:graph:mod:Z}, let 
\begin{align}\label{eq:h:def}
\hv_n :\equiv \prod_{j \in V} P(X_j^n | Z_j^n) \prod_{\{i,j\} \in E} P(X_{ij}^n | Z_i^n , Z_j^n)  
\end{align}
where the ingredients are given by~\eqref{eq:u:potentials}. As before,
in this expression, we are treating $\Zs^n$ as indexing a random
vector in $\reals_+^\mdim$. For $n \in \nats$, let 
\begin{align}
	\thb_n := \frac{\hv_n}{(\hv_n)_{m-1}},
\end{align}
where $(\hv_n)_i$ denotes the $i$th entry of $\hv_n$, using subscript
indexing according to~\eqref{eq:sub:sup:nota}. 
In other words, to obtain $\thb_n$, we normalize $\hv_n = (
(\hv_n)_{m-1},\dots, (\hv_n)_0 )$ by dividing it by its first
entry. Using~\eqref{eq:u:potentials} and~\eqref{eq:bin:rep:def}, we
can write
\begin{align}\label{eq:def:actual:thbn}
	(\thb_n)_\ell = 
		\prod_{j \in V} \Big[ 
		\frac{g_j(X_j^n)}{f_j(X_j^n)}\Big]^{1-\bid_j(\ell)} 
		\prod_{\{i,j\} \in E} \Big[
		\frac{g_{ij}(X_{ij}^n)}{f_{ij}(X_{ij}^n)}
		\Big]^{1-b_j(\ell) \vee b_j(\ell)},
\end{align}
where $\vee$ denotes the maximum.

Recall that for $\rho \in [0,1]$, we use the
notation $\rhob := 1-\rho$.
\subsection{The approximate algorithm follows general
  iteration~\eqref{eq:itr:def}}
\label{sec:proof:approx:alg:rep}
In order to avoid confusion with exact quantities, we will use a tilde to denote the posterior quantities produced by the approximate iteration. For example,~\eqref{eq:mod:prior:Z} can be rewritten as an exact equality in terms of approximate quantities,
\begin{align}\label{eq:mod:prior:Z:2}
\Pt(\Zs^n |\Xbs^{n-1}) = \prod_{j \in V} 
	\nu(Z_j^n; \gamt_j^{n-1}[n])
\end{align}
We first note that  recursion~\eqref{eq:approx:gam:recur} is simplified for a geometric prior. We have $\pi_j(n) = \rhob_j^{n-1} \rho_j$ and $\pi_j [n]^c = \rhob_j^n$. Then,~\eqref{eq:approx:gam:recur} for the approximate algorithm is
\begin{align}\label{eq:approx:gam:recur:2}
	\gamt_j^{n-1}[n] = \rho_j + \rhob_j 
	\gamt_j^{n-1}[n-1].
\end{align}
Consider an operator $\Rf_{\rho}$ on $\pcal_1 := \pcal(\{0,1\})$ defined by 
\begin{align}
\begin{split}\label{eq:Rf:def}
	\Rf_\rho\Big( \Big( {x_1 \atop x_0} \Big) \Big)
	&:= \rho\Big( {1 \atop 0} \Big) + 
	(1-\rho)\Big({x_1 \atop x_0} \Big)
	= \Big( {1 \atop 0} \Big) + 
	(1-\rho)\Big( {-x_0 \atop x_0} \Big)
\end{split}
\end{align}
for any vector $x = (x\li1,x\li0) = (1-x\li0,x\li0) \in
\pcal_1$. (We are using the subscript indexing introduced
in~\eqref{eq:sub:sup:nota}.)


Recall that  $\pcal_d := \pcal(\{0,1\}^d)$. Let  $\Mf_j : \pcal_d \to \pcal_1$ be the $j$th marginalization operator, that is, an operator which produces the $j$-th marginal when applied to probability vector $y \in \pcal_d$. More explicitly, 
\begin{align}\label{eq:marg:expr}
	[\Mf_j(y)]_1 := \sum_{i \,: \;b_j(i) = 1} y_i.
\end{align}
(On the LHS, we are again using the subscript indexing.)
For $z \in \pcal_r$ and $y \in \pcal_d$, let $z \otimes y \in \pcal_{r+d}$ be the probability vector corresponding to the product of $z$ and $y$ as measures. It is the usual tensor product if we think of $z$ and $y$ as vectors.

Now, let
\begin{align*}
\yvt_n :\equiv \Pt(\Zs^n | \Xbs^n), \quad 
\text{and} \quad 
\wvt_n :\equiv \Pt(\Zs^{n} | \Xbs^{n-1})
\end{align*}
in the sense discussed in Section~\ref{sec:mcp:results} leading to~\eqref{eq:equiv:notation}. In words, $\yvt_n$ is a vector in $\pcal_d$ representing the estimate of the joint posterior of $\Zs^n$ given $\Xbs^n$, produced at the $n$-th step of the approximate algorithm. Similar interpretation holds for $\wvt_n$.

Recall that 
$\gamt_j^n[n] = \Pt(Z_j^n = 1| \Xbs^n)$ and $\gamt_j^{n-1}[n] = \Pt(Z_j^n = 1 | \Xbs^{n-1})$. In other words, $(\gamt_j^{n}[n], 1-\gamt_j^{n}[n])$ is the $j$-th marginal of $\yvt_n$, and $(\gamt_j^{n-1}[n], 1-\gamt_j^{n-1}[n])$ is the $j$-th marginal of $\wvt_n$. It follows from~\eqref{eq:approx:gam:recur:2} and the definitions of $\Rf_\rho$ and $\Mf_j$ that
\begin{align*}
	\Mf_j(\wvt_n) = \Rf_{\rho_j}(\Mf_j( \yvt_{n-1})).
\end{align*}
On the other hand,~\eqref{eq:mod:prior:Z:2} states  that $\wvt_n$ is a product measure,
\begin{align*}
	\wvt_n = \otimes_{j=1}^d \Mf_j(\wvt_n).
\end{align*}
Combining the two, we get
\begin{align}\label{eq:Tap:def}
	\wvt_n = \otimes_{j=1}^d \big[ \Rf_{\rho_j}(\Mf_j(
        \yvt_{n-1})) \big] =: \Tap (\yvt_{n-1}).
\end{align}
It is easy to verify that each element of $\Tap(\yvt_{n-1})$ as defined above is a polynomial of degree (at most) $d$ in elements of $\yvt_{n-1}$, with coefficients that depend only on $\{\rho_j\}$.

\medskip
It remains to investigate how $\wvt_n$ produces
$\yvt_n$. Using~\eqref{eq:graph:mod:Z}, we observe that $\wvt_n \equiv
\Pt(\Zs^n | \Xbs^{n-1})$ is mapped to $\Pt(\Zs^n,\Xs^n| \Xbs^{n-1})$
by a pointwise multiplication with $\hv_n$ as defined in~\eqref{eq:h:def}.
Since, $\yvt_n \equiv \Pt(\Zs^n | \Xbs^n)$ is obtained from $\Pt(\Zs^n,\Xs^n| \Xbs^{n-1})$ by a normalization over $\Zs^n$, we obtain
\begin{align}\label{eq:wvt:yvt:update}
	\yvt_n = \frac{\wvt_n \circ \hv_n}{\wvt_n^T \hv_n} 
	= \frac{\wvt_n \circ \thb_n}{\wvt_n^T \thb_n} = \qf_{\thb_n}(\wvt_n). 
\end{align}
This completes the proof.

\subsection{The exact  algorithm follows general
  iteration~\eqref{eq:itr:def}}
\label{sec:proof:exact:alg:rep}
Let 
\begin{align*}
\yv_n :\equiv P(\Zs^n | \Xbs^n), \quad 
\text{and} \quad 
\wv_n :\equiv P(\Zs^{n} | \Xbs^{n-1})
\end{align*}
be the posteriors produced by the exact algorithm. One observes that~\eqref{eq:wvt:yvt:update} holds with $\wvt_n$ replaced with $\wv_n$ and $\yvt_n$ replaced with $\yv_n$. That is, $y_n = q_{\thb_n}(w_n)$. The difference with the approximate algorithm is in updating $\wv_n$ based on $\yv_{n-1}$. To derive this map, we need the following lemma. Recall that $\pi_j$ is the prior on the $j$-th change point $\lam_j$.

\begin{lem}\label{lem:ratio}
	Let $\ical \subset [d]$ and consider collections of integers $\{k_j\}_{j \in \ical}$ and $\{m_j\}_{j \in \ical}$  in $\{n+1,n+2,\dots\}$. Then, we have 
	\begin{align*}
		\frac{P(\lam_j = k_j, j \in \ical|\Xbs^n)}{P(\lam_j = m_j, j \in \ical| \Xbs^n)} = 
		\prod_{j \in \ical} \frac{\pi_j(k_j)}{\pi_j(m_j)}
	\end{align*}
\end{lem}
\begin{proof}
	This follows from Lemma~\ref{lem:const} which implies $P(\lam_j = k_j, j \in \ical|\Xbs^n)$ and $P(\lam_j = m_j, j \in \ical| \Xbs^n)$ are equal for the collection of integers considered.
\end{proof}

We note that both $\wv_n$ and $\yv_{n-1}$ are based on conditional
probabilities, given $\Xbs^{n-1}$, of events in terms of
$\{\lam_j\}$. Updating $\wv_n$ based on $\yv_{n-1}$ amounts to
evaluating the values  a fixed probability measure assigns to a collection of sets, based on the values it assigns to a different collection of sets. The particular nature of these sets and Lemma~\ref{lem:ratio} allow this computation.

The formula has an algebraic structure. We work with polynomials of degree $d$, in indeterminate variables $\omgu$ and $\omgd$. We assume the product of $\omgu$ and $\omgd$ to be noncomutative. (That is, $\omgu \omgd \neq \omgd \omgu$.) Denote the space of such polynomials as $\xcal_d$. We think of $\omgu$ and $\omgd$ as digits $1$ and $0$, respectively. Then, a string consisting of $\omgu$ and $\omgd$ represents a binary number. Let $B(\cdot)$ be the map that produces this binary number given a string of $\omgu$ and $\omgd$. For example, $B(\omgu \omgd \omgu) = 1 0 1 \equiv 5$.

Let $\Lf_{y_{n-1}}(\cdot)$ be a ``linear'' map defined on $\xcal_d$ which maps a string $s$ of $\omgu$ and $\omgd$ to $(\yv_{n-1})_{B(s)}$. This implies, for example,
\begin{align*}
	\Lf_{y_{n-1}}(2 \omgu \omgd\omgu + 3 \omgd\omgu\omgu) = 
	2 (\yv_{n-1})_{5} + 3 (\yv_{n-1})_3.
\end{align*}
Let
\begin{align}\label{eq:def:ui:func}
	u^{(i)}_j(\omgu,\omgd) =
	\begin{cases}
		\rhob_j \omgd, & b_j(i) = 0 \\
		\omgu + \rho_j \omgd & b_j(i) = 1.
	\end{cases}
\end{align}
The following lemma describes the rule mapping $y_{n-1}$ to $w_n$. 
\begin{lem}\label{lem:Tex:L:expr}
  For $i=0,\dots,m-1$,
  \begin{align}\label{eq:yn-1:wn:map}
    (w_n)_i = \Lf_{y_{n-1}} \Big( u_1^{(i)}(\omgu,\omgd) \, u_2^{(i)}(\omgu,\omgd) \,\cdots
     u_d^{(i)}(\omgu,\omgd) \Big).
  \end{align}
\end{lem}
The sketch of the proof is given in Appendix~\ref{sec:proof:Tex:L:expr}. To get a sense of what~\eqref{eq:yn-1:wn:map} means, consider the case $d = 2$. Then, for example, 
\begin{align*}
(w_n)_2 = L_{y_{n-1}}\Big( (\omgu + \rho_1 \omgd) (\rhob_2 \omgd) \Big) &= L_{y_{n-1}}\Big( \rhob_2 \omgu \omgd + \rho_1 \rhob_2 \omgd \omgd \Big)  \\
&= \rhob_2 (y_{n-1})_2 + \rho_1 \rhob_2 (y_{n-1})_0.
\end{align*}
As can be seen from this example,~\eqref{eq:yn-1:wn:map} is a compact
way of expressing a linear relation $w_n = \Tex y_{n-1}$, for some $m
\times m$ matrix $\Tex$. For example, for $d=2$, the matrix is given by
\begin{align}\label{eq:Tex:d:2}
    \Tex = 
        \begin{pmatrix}
            1 & \rho_2 & \rho_1 & \rho_1 \rho_2 \\
            0 & \rhob_2 & 0 & \rho_1 \rhob_2 \\
            0 & 0 & \rhob_1 & \rhob_1 \rho_2 \\
            0 & 0 & 0 & \rhob_1 \rhob_2
        \end{pmatrix}.
\end{align}
This completes the proof.

\subsection{Bounding Lipschitz constant of $\Tex$}

Since $\Tex$ is a Markov transition matrix, we have $\onevec_\mdim^T \Tex = 0$. Note that our convention leads to the transpose of what is usually considered a Markov transition matrix. That is, columns of $\Tex$ sum to $1$ (not the rows). Based on Lemma~\ref{lem:Tex:L:expr}, it is not hard to observe the following:
\begin{itemize}
    \item The first column of $\Tex$ is equal to $\eb\uii{0} := (1,0,\dots,0) \in \reals^m$.
    \item The first row of $\Tex$ consists of elements of the form $\prod_{j \in S} \rho_j$, for $S \subset [d]$. In particular, the first element of the first row is $1$ (corresponding to $S = \emptyset$) while the last element is $\prod_{j=1}^d \rho_j$ (corresponding to $S = [d]$).
\end{itemize}

    We will apply Lemma~\ref{lem:lip:jacob} of Appendix~\ref{app:lip:const:lem} to the linear map $\Ft$ given by $\Ft(x) = \Tex x$ for  $x \in \reals^\mdim$. The Jacobian of $\Tex$ is constant and equal to $\Tex$. Applying Lemma~\ref{lem:lip:jacob}  with $u(x) = (\prod_{j=1}^d \rho_j) \eb\uii{0}$ (independent of $x$), we obtain 
    \begin{align*}
        \lip_\Ft \le \onemnorm{\underbrace{\Tex - (\prod_{j=1}^d \rho_j) \eb\uii{0} \onevec_m^T}_{=: \, A}}.
    \end{align*}
    Note that $\eb\uii{0} \onevec_m^T$ is an $\mdim \times m$ matrix
    with the first row being all ones, and the rest being all
    zeros. Thus, the matrix $A$ coincides with $\Tex$ outside the
    first row. Moreover, on the first row, where $\Tex$ has entry
    $\prod_{j \in S} \rho_j$, $A$ has entry $ \prod_{j \in S} \rho_j -
    \prod_{j =1}^d \rho_j \ge 0$.  That is, all the entries of $A$ are
    nonnegative.  Hence, the absolute column sums for $A$, are the
    same as its column sums. Furthermore, since all the columns of
    both $\Tex$ and $\eb\uii{0} \onevec_m^T$ sum to one, we have
    $\onemnorm{A} = \sum_{i} A_{ik} = 1 - \prod_{j=1}^d \rho_j$, for
    any $k$. This gives the desired bound on the Lipschitz
    constant. (It is not hard to verify that bound is sharp, that is,
    the Lipschitz constant is in fact equal to$1 - \prod_{j=1}^d \rho_j$.)

\subsection{Bounding Lipschitz constant of $\Tap$}
Recall the expression for $\Tap$ given in~\eqref{eq:Tap:def}. We will rewrite it as the composition of two functions. Recall that $\mdim := 2^\ddim$. Let $\Htmp : \reals^\ddim \to \reals^\mdim$ be defined as
\begin{align*}
    \Htmp(u) := \Htmp(u_1,\dots,u_d) :=
    \otimes_{j=1}^d \Big( {u_j \atop 1-u_j}\Big)
\end{align*}
where $\otimes$ is the (tensor) product of two measures defined in Section~\ref{sec:proof:approx:alg:rep}. Here, we use our convention (for embedding $\pcal_\ddim$ in $\reals^\mdim$) to treat the result of the tensor product as an element of $\reals^\mdim$. For example, for $d = 2$, $H(u_1,u_2) = \big(u_1 u_2, u_1(1-u_2), (1-u_1)u_2, (1-u_1)(1-u_2) \big)$.

Also, let $\Ktmp : \reals^\mdim \to \reals^\ddim$ be defined as
\begin{align*}
    \Ktmp(y) := \Big( 
        1-\rhob_1 [\Mf_1(y)]_0, \dots, 1-\rhob_\ddim[\Mf_\ddim(y)]_0 \Big)
\end{align*}
where $[\Mf_j(y)]_0$ is the value assigned to $0$ by the $j$th marginal of $y$. (Note that each marginal $\Mf_j(y)$ is a probability distribution on $\{0,1\}$.) To simplify notation, we will also use
\begin{align*}
    u_j(y) := 1 - \rhob_j [\Mf_j(y)]_0
\end{align*}
so that $\Ktmp(y) = \big(u_1(y),\dots,u_\ddim(y)\big)$.
 For example, for $d = 2$, with $ y = (y_3,y_2,y_1,y_0)$, we have $u_1(y) = 1 - \rhob_1(y_1+y_0)$ and $u_2(y) = 1 -\rhob_2(y_2 + y_0)$.
 
Recalling the definition~\eqref{eq:Rf:def} of $\Rf_{\rho_j}$, and~\eqref{eq:Tap:def}, one observes that $\Htmp \circ \Ktmp := \Htmp(\Ktmp(\cdot))$ is an extension of $\Tap$ to all of $\reals^\mdim$. In other words,
\begin{align*}
    \Tap = \Htmp \circ \Ktmp \big|_{\pcal_\ddim}.
\end{align*}
Thus, we can estimate the Lipschitz constant of $\Tap$ by computing
the Jacobian of $\Htmp \circ \Ktmp$ and applying
Lemma~\ref{lem:lip:jacob} of Appendix~\ref{app:lip:const:lem}. By chain rule, the Jacobian of the  composition is the product of Jacobians. More precisely, $\jacob_{\Htmp \circ \Ktmp}(y) = \jacob_\Htmp(u) \jacob_\Ktmp(y)$ with $u = \Ktmp(y)$.

To compute $\jacob_\Htmp(u) \in \reals^{\mdim \times \ddim}$, first note that we can write the $i$th component of $\Htmp(u)$ as $[\Htmp(u)]_i = \prod_{k=1}^\ddim u_k^{b_k(i)} (1 - u_k)^{1-b_k(i)}$ where $b_k(i)$ is the bit notation introduced in~\eqref{eq:bin:rep:def}. It follows that
\begin{align*}
    [\jacob_\Htmp(u)]_{ij} = \partial_{u_j} [\Htmp(u)]_i = (-1)^{1 - b_j(i)} \prod_{k\neq j} u_k^{b_k(i)} (1 - u_k)^{1-b_k(i)}
\end{align*}
For $y \in \pcal_\ddim$, we have $ u = \Ktmp(y) \in [0,1]^\ddim$, that is, both $u_k$ and $1-u_k$ are nonnegative for all $k \in [\ddim]$. It is not then hard to verify that
$ \sum_{i=1}^\mdim \big| [\jacob_\Htmp(u)]_{ij} \big| = 2$,
for all $j \in [d]$. That is, all the absolute column sums of $\jacob_\Htmp$ are equal to $2$, which implies $\onemnorm{\jacob_\Htmp(u)} = 2$ for $u \in [0,1]^\ddim$.

Turning to $\jacob_\Ktmp(y) \in \reals^{\ddim \times \mdim}$, we note that this is in fact a constant matrix, as $\Ktmp$ is an affine map. Using an expression similar to~\eqref{eq:marg:expr}, we have
\begin{align*}
    [\jacob_\Ktmp]_{j \ell} = \partial_{y_\ell} u_j = - \rhob_j 
    \partial_{y_\ell} \Big(
    \sum_{i :\, b_j(i) = 0} y_i \Big) = -\rhob_j (1 - b_j(\ell)).
\end{align*}
In other words, the $j$-th row of $\jacob_\Ktmp$ contains $-\rhob_j$ in columns $\ell$ with $b_j(\ell) = 0$, and is zero otherwise. For example, for $\ddim = 3$ (and $\mdim = 8$), we obtain
\begin{align*}
    \jacob_\Ktmp = -
    \begin{pmatrix}
        0 & 0 & 0 & 0 & \rhob_1 & \rhob_1 & \rhob_1 & \rhob_1 \\
        0 & 0 & \rhob_2 & \rhob_2 & 0 &0 & \rhob_2 & \rhob_2 \\
        0 & \rhob_3 & 0 & \rhob_3 & 0 & \rhob_3 & 0 & \rhob_3
    \end{pmatrix}
\end{align*}
According to Lemma~\ref{lem:lip:jacob}, it is possible to add a constant to each row of $\jacob_\Ktmp$ and still obtain an upper bound on the Lipschitz  constant of $\Tap$. We will add $\rho_j /2$ to each column in the $j$-th row. More precisely, let $\rb := (\rhob_1,\dots,\rhob_\ddim) \in \reals^\ddim$. Then, we consider $\jacob_\Ktmp + \frac12 \rb \onevec_\mdim^T$. For example, in the case of $d = 3$, we have
\begin{align*}
    \jacob_\Ktmp +  \frac12 \rb \onevec_\mdim^T= \frac12
    \begin{pmatrix}
        \rhob_1 & \rhob_1 & \rhob_1 & \rhob_1 & -\rhob_1 & -\rhob_1 & -\rhob_1 & -\rhob_1 \\
        \rhob_2 & \rhob_2 & -\rhob_2 & -\rhob_2 & \rhob_2 & \rhob_2 & -\rhob_2 & -\rhob_2 \\
        \rhob_3 & -\rhob_3 & \rhob_3 & -\rhob_3 & \rhob_3 & -\rhob_3 & \rhob_3 & -\rhob_3
    \end{pmatrix}.
\end{align*}
It is easy to verify that the absolute column sum for each column of
this new matrix equal to $\frac12 \sum_{j=1}^\ddim \rhob_j$. That is,
$\onemnorm{\jacob_\Ktmp +  \frac12 \rb \onevec_\mdim^T} = \frac12
\sum_{j=1}^\ddim \rhob_j$.  

We can now apply lemma~\ref{lem:lip:jacob} to obtain
\begin{align*}
    \lip_{\Tap} &\le \sup_{y \in \pcal_\ddim} \onemnorm{ \jacob_{\Htmp \circ \Ktmp}(y) + \big( \frac12 \jacob_\Htmp(u) \rb \big) \onevec_\mdim^T } \\
    &= \sup_{y \in \pcal_\ddim} \onemnorm{ \jacob_{\Htmp}(u)  \big[ \jacob_\Ktmp +  \frac12 \rb \onevec_\mdim^T  \big]} \\
    &\le \sup_{y \in \pcal_\ddim} \Big\{ \onemnorm{ \jacob_{\Htmp}(u)}\, \onemnorm{ \jacob_\Ktmp +  \frac12 \rb \onevec_\mdim^T  } \Big\} = 
   \sum_{j=1}^\ddim \rhob_j
\end{align*}
where as before $u = \Ktmp(y)$, and the last inequality follows by the sub-multiplicative property of $\onemnorm{\cdot}$. The proof is complete.




\bibliographystyle{unsrt}
\bibliography{semi_irf}

\appendix

\section{Proof of~\eqref{eq:gamma:classic:recur}}
\label{sec:proof:eq:gamma:classic:recur}
Recall that $\pi(k) := \pr(\lambda = k)$. Let $[n] := \{1,\dots,n\}$
and $[n-1]^c :=
\{n,n+1,\dots\}$. For $k,r \in [n-1]^c$, we have
\begin{align}\label{eq:ratio:recur:simp}
  \frac{\pr(\lambda = k|\Xb^{n-1})}{\pr(\lambda = r|\Xb^{n-1})}
  = \frac{P(\Xb^{n-1}|\lambda = k) \pi(k)}{ P(\Xb^{n-1}|\lambda = r) \pi(r)} = \frac{\pi(k)}{\pi(r)}
\end{align}
since the function $k \mapsto P(\Xb^{n-1} | \lambda = k)$ is constant over $[n-1]^c$. In fact, $P(\Xb^{n-1} | \lambda = k) =
\prod_{t=1}^{n-1}g(X^t)$ for all $k \ge
n$. In~\eqref{eq:ratio:recur:simp}, take $ r = n$, and sum over $k \in
[n-1]^c$ to obtain (after inversion)
\begin{align*}
  \frac{\pr(\lambda = n|\Xb^{n-1})}{\pr(\lambda \in [n-1]^c |\Xb^{n-1})}
  = \frac{\pi(n)}{\pi[n-1]^c}.
\end{align*}
For any subset $A \subset \nats := \{1,2,\dots\}$, let us use the
notation $\gam^{n-1}A := \pr (\lambda \in A|\Xb^{n-1})$. Thus, we have
shown $\gam^{n-1}\{n\} = \frac{\pi(n)}{\pi[n-1]^c} \gam^{n-1}[n-1]^c$.

From additivity of probability measures, we have  
\begin{align*}
  \gam^{n-1}[n-1]^c = 1-\gam^{n-1}[n-1], \quad \gam^{n-1}\{n\} = \gam^{n-1}[n] -  \gam^{n-1}[n-1].
\end{align*}
Substituting these in the earlier equation, we obtain
\begin{align*}
  \gam^{n-1}[n]  =  \frac{\pi(n)}{\pi[n-1]^c} + 
  \Big(1- \frac{\pi(n)}{\pi[n-1]^c}  \Big)\gam^{n-1}[n-1]
\end{align*}
which is the desired result.

\section{Proof of Lemma~\ref{lem:Tex:L:expr}}
\label{sec:proof:Tex:L:expr}
Let $\nats := \{1,2,\dots\}$ denote the set of natural numbers. Let $A:= [n]:= \{1,\dots,n\}$ and let $A^c$ be the complement of $A$ in $\nats$, that is, $A^c = \{n+1,n+2,\dots\}$. Similarly, let $B = [n+1]$ and let $B^c = \{n+2,n+3,\dots\}$. We also let $b := \{n+1\}$. (These notations are local to this proof.)

For an index set $\ical = \{i_1,\dots,i_r\} \subset d$, let $\gam_{\ical}^n$ denote the joint posterior of $\lam_\ell, \ell \in \ical$ given $\Xbs^n$. More precisely, $\gam_{\ical}^n(E_1,\dots,E_r) = \pr(\bigcap_{j=1}^r \{\lam_{i_j} \in E_j\} | \Xbs^n)$ for any collection $E_1,\dots, E_r$ of subsets of $\nats$.  Let $A^\circ$ denote either $A$ or $A^c$, and similarly for $B^\circ$. We would like to compute quantities of the form $\gam_\ical^n(B^\circ,\dots,B^\circ)$ in terms of known quantities $\gam_\ical^n(A^\circ,\dots,A^\circ)$. For simplicity, we will drop superscript $n$ from now on.

We will use $-$ and $+$ to denote set difference and disjoint union, respectively. For example, $B = A + b$ and $B^c = A^c - b$. We proceed in stages, by first finding probabilities of ``sequences of $A^c$ and $b$''; we do this by an example. Consider $\gam_{1234}(A^c,b,A^c,b)$. Applying Lemma~\ref{lem:ratio}, we have
\begin{align*}
    \frac{\gam_{1234}(A^c,b,A^c,b)}{\gam_{1234}(b,b,b,b)} 
    = \frac{\pi_1(A^c)}{\pi_1(b)} \frac{\pi_3(A^c)}{\pi_3(b)} = \frac{1}{\rho_1 \rho_3}.
\end{align*}
Similarly, 
\begin{align*}
    \frac{\gam_{1234}(A^c,A^c,A^c,A^c)}{\gam_{1234}(b,b,b,b)} 
    = \frac{1}{\rho_1 \rho_2 \rho_3 \rho_4}.
\end{align*}
It follows that
\begin{align*}
    \gam_{1234}(A^c,b,A^c,b) = \rho_2 \rho_4 \,\gam_{1234}(A^c,A^c,A^c,A^c),
\end{align*}
which is the desired result, since the RHS is known. By induction, we have the following rule: The probability of a sequence of $A^c$ and $b$ is the probability of the sequence of all-$A^c$ multiplied by ``$\rho_i$''s associated with places of ``$b$''s. We will later use a more compact notation: $A^c b A^c b = \rho_2 \rho_4 A^c A^c A^c A^c$, to express the same fact.

We turn to the case where we have a sequence of $A^c$ and $b$ an a single $A$. Consider, for example,
\begin{align*}
    \gam_{1234}(A^c,b,A,b) &= \gam_{1234}(A^c,b,\nats,b) - \gam_{1234}(A^c,b,A^c,b) \\
    &= \gam_{124}(A^c,b,b) - \gam_{1234}(A^c,b,A^c,b) \\
    &= \rho_{2} \rho_4\, \gam_{124}(A^c,A^c,A^c) - 
        \rho_2 \rho_4 \,\gam_{1234}(A^c,A^c,A^c,A^c) \\
    &= \rho_2 \rho_4 \, \gam_{1234}(A^c,A^c,A,A^c).
\end{align*}
where third equality follows by the rule regarding sequences of $A^c$ and $b$. Thus, by induction, we can revise our rule to include the sequences with a single $A$: We proceed by replacing ``$b$''s with $A^c$ and multiplying by corresponding ``$\rho_i$''s, leaving the $A$ intact.

Now, consider a sequence with more than one $A$. For example,
\begin{align*}
    \gam_{1234}(A^c,b,A,A) &= 
    \gam_{1234}(A^c,b,A,\nats) - \gam_{1234}(A^c,b,A,A^c) \\
    &= \gam_{123}(A^c,b,A) - \gam_{1234}(A^c,b,A,A^c) 
\end{align*}
where both terms involve sequences with single $A$. Applying our rule to each term and combining the result as before, we get, in compact notation, $A^c b A A = \rho_2 A^c A^c A A$. Thus, by induction, our rule extends to sequences of $A^c$, $b$, and arbitrary number of ``$A$''s: Replace ``$b$''s with ``$A^c$''s and scale appropriately, leaving ``$A$''s intact.

We are now ready to obtain probabilities of a sequence of $B$s and $B^c$s. Consider the following example,
\begin{align*}
    \gam_{12}(B^c,B) &= \gam_{12}(A^c - b, A+ b) \\
    &= \gam_{12}(A^c - b, A) + \gam_{12}(A^c -b,b) \\
    &= \gam_{12}(A^c,A) - \gam_{12}(b,A) + \gam_{12}(A^c,b) - \gam_{12}(b,b),
\end{align*}
by finite additivity of probability measures. We can represent this identity in a compact form. $B^cB =(A^c -b)(A+b) = A^c A - bA + A^cb - bb$. Applying our rule, we obtain
\begin{align*}
    B^c B &= 
    A^c A - \rho_1 A^c A + \rho_2 A^c A^c - \rho_1 \rho_2 A^c A^c \\
    &= (1-\rho_1) A^c A + (1-\rho_1)\rho_2 A^c A^c.
\end{align*}
This result can be obtained easier by replacing $b$ in the first and the second sets of parentheses with $\rho_1 A^c$ and $\rho_2 A^c$, respectively, and following rules of a noncommutative associative algebra,
\begin{align*}
    B^c B &= (A^c - b)(A+b) \\
    &= (A^c - \rho_1 A^c) (A + \rho_2 A^c) = (1-\rho_1)A^c(A+\rho_2 A^c) = 
    \rhob_1 A^c A + \rhob_1 \rho_2 A^c A^c.
\end{align*}
Using this procedure, we can express the probability of any sequence of $B$ and $B^c$ in terms of sequences of $A$ and $A^c$. As another example,
\begin{align}
    B^c B B B^c &= (A^c -b)(A+b)(A+b)(A^c - b) \notag \\
    &= (A^c - \rho_1 A^c)(A+ \rho_2 A^c) (A+\rho_3 A^c) (A^c - \rho_4A^c) \notag \\
    &= (\rhob_1 A^c) (A+ \rho_2 A^c) (A+\rho_3 A^c) (\rhob_4 A^c). 
    \label{eq:}
\end{align}
As before, the final expression is obtained by expanding. The general pattern is now clear and can be formally established by induction. The proof is complete. To link with the notation of the theorem, replace $A^c$ with $\omgd$ and $A$ with $\omgu$. The function $u^{(i)}$ defined in~\eqref{eq:def:ui:func} replaces a set of parantheses, in derivations above, with the correct expression in terms of $\omgd$ and $\omgu$, depending on whether the set of parantheses contains a $+$ or a $-$ sign. 

\section{Bounding the Lipschitz constant of a probability map}
\label{app:lip:const:lem}
This appendix is devoted to a lemma which allows us to estimate the
Lipschitz constant of a map $F : \pcal \to \pcal$, on a probability
space $\pcal$, based on the Jacobian matrix of its extension. Here, $\pcal := \pcal_\ddim := \pcal(\{0,1\}^\ddim)$ is considered to be a subset of $\reals^\mdim$ where $\mdim = 2^\ddim$. For a $C^1$ function $\Ft : U \to \reals^\mdim$ defined on some open subset $U$ of $\reals^\mdim$, let $\jacob_\Ft$ denote its Jacobian matrix, i.e.,
\begin{align*}
    \jacob_\Ft := \big( \partial_{x_j} \Ft_i \big) \in \reals^{\mdim \times \mdim}
\end{align*}
where $\partial_{x_j} \Ft_i$ is the partial derivative
of the $i$-th component of $\Ft$ w.r.t. the its $j$-th
variable.  

For a square matrix $A$ and $p \in [1,\infty]$, let $\mnorm{A}_p$ denote its
norm as an operator on $\ell_p$, that is, $\mnorm{A}_p :=
\sup_{\|x\|_{p} \le  1} \| A x\|_p$, where $\|\cdot\|_p$ is the vector
$\ell_p$ norm. It is well-known that $\mnorm{A}_1$ (
$\mnorm{A}_\infty$) is the maximum absolute column (row) sum of 
matrix $A$.

Recall that $\onevec_\mdim \in \reals^\mdim$ denotes the all-ones vector.

\begin{lem}\label{lem:lip:jacob}
    Let $U$ be an open subset of $\reals^\mdim$, containing $\pcal$. Let $\Ft : U \to \reals^\mdim$ be a $C^1$ extension of $F : \pcal \to \pcal$, that is, $\Ft \mid_\pcal = F$. Then, for any function $u : U \to \reals^\mdim$ with components in $L^1(U)$,
    \begin{align}\label{eq:lip:jacob}
        \lip_\Ft \le \;\sup_{x \in \pcal} 
        \onemnorm{\jacob_\Ft(x) - u(x) \onevec_\mdim^T}.
    \end{align}
\end{lem}
\begin{proof}
    Fix some $x,y \in \pcal$ and let $z_t := x + t(y-x)$ for $t \in [0,1]$. For $v \in \reals^\mdim$, we have
     \begin{align*}
         v^T \big(\Ft(y) - \Ft(x) \big) &=
         \int_0^1 v^T \frac{d}{dt} \Ft\big(x + t(y-x) \big) \, dt \\
         &= \int_0^1 v^T \jacob_\Ft(z_t) (y-x) \,dt \\
         &= \int_0^1 v^T \underbrace{\big[ \jacob_\Ft(z_t)  -u(z_t) \onevec_m^T\big]}_{=: \,R_t^T}(y-x) \,dt
      \end{align*} 
      where the last line follows since $x,y \in \pcal$ implies $\onevec_\mdim^T(y-x) = 0$. Using $\ell_1$--$\ell_\infty$ duality, we have
      \begin{align*}
          \onenorm{\Ft(y) - \Ft(x)} = \sup_{\infnorm{v} \le 1} \big|
          v^T \big(\Ft(y) - \Ft(x) \big) \big|
          &\le \int_0^1 \sup_{\infnorm{v} \le 1} 
              \big| (R_tv)^T (y -x)\big|\, dt \\
          &\le \onenorm{y-x} \int_0^1 \sup_{\infnorm{v} \le 1} \infnorm{R_t
          v} \, dt \\
          &= \onenorm{y-x} \int_0^1 \infmnorm{R_t} \, dt.
      \end{align*}
      Let us denote the RHS of~(\ref{eq:lip:jacob}) by $L$. Since $z_t
      \in \pcal$ for all $t \in [0,1]$, we have $\infmnorm{R_t} =
      \onemnorm{R_t^T} \le L$, for all $t \in [0,1]$, which completes
      the proof. 
\end{proof}

\section{An auxiliary lemma}
Here, we record the following ``constancy'' property of the likelihood 
for the graphical model~\eqref{eq:joint:def}. See~\cite[Lemma~3]{AmiNgu13} for
the proof.
\begin{lem}\label{lem:const}
  Let $\{i_1,i_2,\dots,i_r\} \subset [d]$ be a distinct collection of
  indices. The function 
  \begin{align*}
   (k_1,k_2,\dots,k_r) \mapsto
  P(\Xbs^n|\lam_{i_1}=k_1, \lam_{i_2}=k_2,\dots,\lam_{i_r} = k_r) 
  \end{align*}
   is  constant over $\{n+1,n+2,\dots\}^r$.
\end{lem}


\end{document}